\documentclass[12pt, a4paper]{article}
\usepackage{amsmath,amsfonts,amsthm,amssymb,dsfont,natbib}
\usepackage[utf8]{inputenc}
\usepackage[T1]{fontenc}
\usepackage[colorlinks=true,linkcolor=blue, citecolor=blue,urlcolor=blue]{hyperref}
\usepackage{csquotes}
\usepackage{float}
\usepackage{graphicx}

\theoremstyle{plain}
\newtheorem{theorem}{Theorem}[section]

\newtheorem{corollary}[theorem]{Corollary}
\newtheorem{proposition}[theorem]{Proposition}

\newtheorem{definition}[theorem]{Definition}
\newtheorem{assumption}[theorem]{Assumption}
\theoremstyle{definition}
\newtheorem{example}[theorem]{Example}
\newtheorem{remark}[theorem]{Remark}

\newcommand{\rmd}{\mathrm{d}}

\begin{document}

\title{Behavior of linear $L^2$-boosting algorithms in the vanishing learning rate asymptotic}
\author{Clément Dombry$^*$ and Youssef Esstafa$^{*}$\footnote{Universit\'e Bourgogne Franche-Comt{\'e}, Laboratoire de Math{\'e}matiques de Besan\c{c}on UMR6623, CNRS, F-25000 Besan{\c c}on, France. Email: \href{mailto:clement.dombry@univ-fcomte.fr}{clement.dombry@univ-fcomte.fr}
\hspace*{0.35cm}$^{**}$\hspace*{0.04cm}ENSAI, Campus de Ker-Lann, 51 Rue Blaise Pascal, BP 37203 - 35172 Bruz Cedex, France. Email: \href{mailto:youssef.esstafa@ensai.fr}{youssef.esstafa@ensai.fr}}}
\maketitle
\date{}

\begin{abstract}
We investigate the asymptotic behaviour of gradient  boosting algorithms when the learning rate converges to zero and the number of iterations is rescaled accordingly. We mostly consider $L^2$-boosting for regression with linear base learner as studied in  \cite{BY03} and analyze also a stochastic version of the model where subsampling is used at each step \citep{F02}. We prove a deterministic limit in the vanishing learning rate asymptotic and characterize the limit as the unique solution of a linear differential equation in an infinite dimensional function space. Besides, the training and test error of the limiting procedure are thoroughly analyzed. We finally illustrate and discuss our result on a simple numerical experiment where the linear $L^2$-boosting operator is interpreted as a smoothed projection and time is related to its number of degrees of freedom.
\end{abstract}

\noindent
\textbf{Keywords:} boosting, non parametric regression, statistical learning, stochastic algorithm, Markov chain, convergence of stochastic process. \\
\textbf{Mathematics subject classification:} 62G08, 60J20.

\pagebreak

\tableofcontents

\section{Introduction}
In the past decades, boosting has become a major and powerful prediction method in machine learning. The success of the classification algorithm AdaBoost  by \cite{FS99} demonstrated the possibility to combine many weak learners in a sequential way in order to produce better predictions, with widespread  applications in gene expression \citep{Dudoit2002} or music genre identification \citep{Bergstra2006}, to name only a few. \cite{FHT00} were able to see a wider statistical framework that lead to the gradient boosting  \citep{F01}, where a weak learner (e.g., regression trees) is used to optimize a loss function in a sequential procedure akin to gradient descent. Choosing the loss function according to the statistical problem at hand results in a versatile and efficient tool that can handle classification, regression, quantile regression or survival analysis... The popularity of gradient boosting is also due to its efficient implementation in the R package \texttt{gbm} by \cite{R07}. 

Along the methodological developments, strong theoretical results have justified the good performance of boosting. Consistency of boosting algorithm, i.e. their ability to achieve the optimal Bayes error rate for large samples, is considered in \cite{B04}, \cite{ZY05} or \cite{BT07}. The present paper is strongly influenced by \cite{BY03} that proposes an analysis of regression boosting algorithms built on linear base learners thanks to explicit formulas for the boosted predictor and its error rate. 

In this paper, we focus  on gradient boosting for regression with square loss and we briefly describe the corresponding algorithm. Consider a regression model 
\begin{equation}\label{eq:regression-model}
Y=f(X)+\varepsilon
\end{equation}
where the response $Y$ is real-valued, the predictor $X$ takes values in $[0,1]^p$, the regression function $f:[0,1]^p\to \mathbb{R}$ is measurable and the error $\varepsilon$ is centered, square integrable and independent of $X$. Based on  a sample  $(Y_i,X_i)_{1\leq i\leq n}$ of independent observations of the regression model \eqref{eq:regression-model}, we aim at estimating the regression function $f$. Given a weak learner $L(x)=L(x;(Y_i,X_i)_{1\leq i\leq n})$, the boosting algorithm with learning rate $\lambda\in(0,1)$ produces a sequence of models $\hat F_m^\lambda(x)$, $m\geq 0$, by recursively fitting the weak learner to the current residuals and updating the model with a shrunken version of the fitted model. More formally, we define
\begin{equation}\label{eq:boosting-recursion}
\left\{\begin{array}{ll}
\hat F_0^\lambda(x)&= \bar Y_n, \\
\hat F_{m+1}^\lambda(x)&=\hat F_m^\lambda(x)+\lambda L(x;(R_{m,i}^\lambda,X_i)_{1\leq i\leq n}),\quad m\geq 0,
\end{array}\right.
\end{equation}
where $\bar Y_n$ denotes the empirical mean of $(Y_i)_{1\leq i\leq n}$ and $(R_{m,i}^\lambda)_{1\leq i\leq n}$ the residuals 
\[
R_{m,i}^\lambda=Y_i-\hat F_m^\lambda(X_i),\quad 1\leq i\leq n.
\]

In practice, the shrinkage parameter $\lambda$ and the number of iterations $m$ are the main parameters and must be chosen suitably to achieve good performance. Common practice is to fix $\lambda$ to a small value, typically $\lambda=0.01$ or $0.001$, and then to select $m$ by cross-validation. Citing \cite{R07}, with slight modifications to match our notations:
\begin{displayquote}
\textit{"The issues that most new users of gbm struggle with are the choice of tree numbers $m$ and shrinkage $\lambda$. It is important to know that smaller values of $\lambda$ (almost) always give improved predictive performance. That is, setting $\lambda=0.001$ will almost certainly result in a model with better out-of-sample predictive performance than setting $\lambda=0.01$. However, there are computational costs, both storage and CPU time, associated with setting shrinkage to be low. The model with $\lambda=0.001$ will likely require ten times as many iterations as the model with $\lambda=0.01$, increasing storage and computation time by a factor of $10$."}
\end{displayquote}

This citation clearly emphasizes the role of small learning rates in boosting. The purpose of the present paper is to prove the existence of a vanishing learning rate limit ($\lambda\to 0$) for the boosting algorithm when the number of iterations is rescaled accordingly. To our best knowledge, this is the first result in this direction. More precisely, in the case when the base learner is linear, we prove the existence of the limit
\begin{equation}\label{eq:limit}
\hat F_t(x)=\lim_{\lambda\downarrow 0} \hat F_{[t/\lambda]}^\lambda(x),\quad t\geq 0.
\end{equation}
We furthermore characterize the limit as the solution of a linear differential equation in infinite dimensional space and also analyse the corresponding training and test errors. The case of stochastic gradient boosting \citep{F02}, where subsampling is introduced at each iteration, is also analysed: we prove the existence of a deterministic  vanishing learning rate limit that corresponds to a modified deterministic base learner defined in a natural way. The analysis of this stochastic framework requires involved tools of Markov chain theory and the characterization of their convergence through generators \citep{EK86, SV06}. A limitation of our work is the strong assumption of linearity of the base learner:  the ubiquitous regression tree does not satisfy this assumption and further work is needed to deal with this important case. Our results are of probabilistic nature: we focus on the existence and properties of the limit \eqref{eq:limit} for fixed sample size $n\geq 1$, while statistical issues such as consistency as $n\to\infty$ is left aside for further research. 

The paper is structured as follows. In Section~\ref{sec:L2boosting-BY}, we prove the existence of the vanishing learning rate limit \eqref{eq:limit} for the boosting procedure with linear base learner (Proposition~\ref{prop:vanishing-learning-rate}), we characterize the limit as the solution of a linear differential equation in a function space (Theorem~\ref{thm:EDO}) and we analyze the training and test errors (Propositions~\ref{prop:train-error} and~\ref{prop:test-error-2}). The stochastic gradient boosting where subsampling is introduced at each step is considered in Section~\ref{sec:SGB}. We prove that the vanishing learning rate limit still exists and that the convergence holds in quadratic mean (Corollary~\ref{cor:vanishing-limit-stochastic}) and also in the sense of functional weak convergence in Skorokhod space (Theorem~\ref{thm:Markov-functional}).   A simple numerical experiment is presented in Section~\ref{sec:illustration} in order to illustrate our theoretical findings, leading us to the interpretation of linear $L^2$-boosting as a smoothed projection where time is related to the degrees of freedom of the linear boosting operator. All the technical proofs are gathered in Section~\ref{sec:proofs}.

\section{$L^2$-boosting with linear base learner}\label{sec:L2boosting-BY}
\subsection{Framework}
We consider the framework of boosting for regression with $L^2$-loss and linear base learner provided by \cite{BY03}. This framework allows for explicit computations relying on linear algebra. The regression design is assumed deterministic, or equivalently, we formulate our results conditionally on the predictor values  $X_i=x_i$, $i=1,\ldots,n$.  The space of measurable and bounded functions on $[0,1]^p$ is denoted by $L^\infty=L^\infty([0,1]^p,\mathbb{R})$. Our main hypothesis is the following linearity assumption of the base learner $L$.

\begin{assumption}\label{ass:BY}
We assume that the base learner of the boosting algorithm~\eqref{eq:boosting-recursion} satisfies
\begin{equation}\label{eq:linear-assumption}
L(x;(x_i,Y_i)_{1\leq i\leq n})= \sum_{j=1}^n Y_jg_j(x),\quad x\in[0,1]^p,
\end{equation}
where $g_1,\ldots,g_n\in L^\infty$ may depend on $(x_i)_{1\leq i\leq n}$.
\end{assumption}
It follows from Assumption~\ref{ass:BY} that $g_j$ is the output of the base learner for input $(Y_i)_{1\leq i\leq n}=(\delta_{ij})_{1\leq i\leq n}$, where the Kroenecker symbol $\delta_{ij}$ is equal to $1$ if $i=j$ and $0$ otherwise. 

Under Assumption~\ref{ass:BY}, the boosting algorithm with input $(Y_i,x_i)_{1\leq i\leq n}$ and learning rate $\lambda\in (0,1)$ outputs a sequence of bounded functions $(\hat F_m^\lambda)_{m\geq 1}$. The sequence remains in the finite dimensional linear space spanned in  $L^\infty$ by  the functions $g_1,\ldots,g_n$ and the constant functions (due to the initialization equal to the constant function $\bar Y_n$). A straightforward recursion based on Equation \eqref{eq:boosting-recursion} yields
\begin{equation}\label{eq:linear-formula}
\hat F_m^\lambda(x)=\bar Y_n+\sum_{i=1}^n w_{m,i}^\lambda g_i(x)
\end{equation}
where the weights $w_{m}^\lambda=(w_{m,i}^\lambda)_{1\leq i\leq n}$ satisfy
\[
\left\{\begin{array}{ll}
w_{0,i}^\lambda&\equiv 0  \\
w_{m+1,i}^\lambda &= w_{m,i}^\lambda +\lambda(Y_i-\bar Y_n)-\lambda\sum_{j=1}^nw_{m,j}^\lambda g_j(x_i)
\end{array}\right..
\]
This linear recursion system can be rewritten in vector form as
\begin{equation}\label{weights-vec}
\left\{\begin{array}{ll}
w_{0}^\lambda&\equiv 0 \\
w_{m+1}^\lambda &= (I-\lambda S)w_{m}^\lambda+\lambda \tilde Y
\end{array}\right.,
\end{equation}
with  $S=(g_j(x_i))_{1\leq i,j\leq n}$, $\tilde Y=(Y_i-\bar Y_n)_{1\leq i\leq n}$  the  centered observations and $I$ the $n\times n$ identity matrix. This linear recursion is easily solved, yielding the following proposition. 

\begin{proposition}\label{prop:weight-formula}
Under Assumption~\ref{ass:BY}, the boosting algorithm output $\hat F_m^\lambda$ is given by Equation~\eqref{eq:linear-formula} with weights 
\begin{equation}\label{eq:weight}
w_{m}^\lambda=\lambda\sum_{j=0}^{m-1}(I-\lambda S)^j\tilde Y,\quad m\geq 0.
\end{equation}
If the matrix $S$ is invertible, then 
\[
w_m^\lambda=S^{-1}\left[ I-(I-\lambda S)^m\right] \tilde Y,\quad m\geq 0.
\]
\end{proposition}

Note that this result is similar to Proposition~1 in \citet{BY03}, but they consider only the values on the observed sample $(x_i)_{1\leq i\leq n}$ while we provide the extrapolation to $x\in[0,1]^p$ more explicitly. Also we consider a different initialization to the empirical mean  instead of initialization to zero, which seems more
relevant in practice.

\begin{example}
A simple example satisfying Assumption~\ref{ass:BY} is the Nadaraya-Watson estimator (see \cite{Nadaraya1964} and \cite{Watson1964})
\[
L(x)=\frac{\sum_{i=1}^n K_h(x-x_i)Y_i}{\sum_{i=1}^n K_h(x-x_i)}, \quad x\in[0,1]^p,
\]
where $h>0$ is the bandwidth,  $K:\mathbb{R}^p\to (0,+\infty)$ is the kernel, i.e. a density function, and $K_h(z)=h^{-d}K(z/h)$ the rescaled kernel.
\end{example}
\begin{example}\label{ex:smoothing-spline}
A more involved example of base learner, discussed in \cite{BY03} Section~3.2, is the smoothing spline in dimension $p=1$. For $r\geq 1$ and $\nu>0$, the smoothing spline $L$ is the unique minimizer over  $\mathcal{W}_2^{(r)}$  of the penalized criterion
\[
\sum_{i=1}^n (Y_i-L(x_i))^2+\nu \int_{0}^1 (L^{(r)}(x))^2 \rmd x,
\]
where $\mathcal{W}_2^{(r)}$ denotes the Sobolev space of functions that are continuously differentiable of order $r-1$  with square integrable weak derivative of order $r$. Assuming $0<x_1<\cdots<x_n<1$, the solution is known to be piecewise polynomial function of degree $r+1$ with constant derivative of order $r+1$ on $n+1$ intervals $(0,x_1),\ldots,(x_n,1)$.
It is used in \cite{BY03} that the matrix $S$ is symmetric definite positive with positive eigenvalues $1=\mu_1=\ldots=\mu_r>\ldots >\mu_n>0$, see \cite{wahba1990}.
\end{example}

\subsection{The vanishing learning rate asymptotic}
We next consider the existence of a limit in the vanishing learning rate asymptotic $\lambda\to 0$.  The explicit simple formulas from Proposition~\ref{prop:weight-formula} allows for a simple analysis. We recall that the exponential of a square matrix $M$ is defined by
\[
\exp(M)=\sum_{k\geq 0} \frac{1}{k!}M^k.
\]

\begin{proposition}\label{prop:vanishing-learning-rate}
Under Assumption~\ref{ass:BY}, as $\lambda\to 0$, we have 
\begin{equation}\label{eq:lambda-to-0}
\hat F_{[t/\lambda]}^\lambda(x)\longrightarrow \hat F_{t}(x), \quad t\geq 0,\ x\in[0,1]^p,
\end{equation}
uniformly on compact sets $[0,T]\times [0,1]^p$, $T>0$, where the limit satisfies
\begin{equation}\label{eq:Ft}
\hat F_t(x)=\bar Y_n+\sum_{i=1}^n w_{t,i} g_i(x)
\end{equation}
with weights $w_t=(w_{t,i})_{1\leq i\leq n}$ given by
\begin{equation}\label{eq:wt}
w_t=-\sum_{j\geq 1}\frac{(-t)^j}{j!}S^{j-1} \tilde Y,\quad t\geq 0.
\end{equation}
If the matrix $S$ is invertible, then 
\begin{equation}\label{eq:wt-inv}
w_t=S^{-1}\left( I-e^{-tS}\right) \tilde Y,\quad t\geq 0.
\end{equation}
\end{proposition}

The formulas are even more explicit in the case when $S$ is a  symmetric matrix because it can then  be diagonalized in an orthonormal basis of eigenvectors.
\begin{corollary}\label{cor:symmetric}
Suppose Assumption~\ref{ass:BY} is satisfied and  $S=(g_j(x_i))_{1\leq i,j\leq n}$ is a symmetric matrix. Denote by $(\mu_j)_{1\leq j\leq n}$ the eigenvalues of $S$ and by $(u_j)_{1\leq j\leq n}$ the corresponding eigenvectors. Then the vanishing learning rate asymptotic yields the
weights
\[
w_t=\sum_{j=1}^n \frac{1-e^{-\mu_j t}}{\mu_j}u_j u_j^{\mathrm{T}}\tilde Y
\]
and the limit
\begin{equation}\label{F-exp-limit}
\hat F_t(x)=\bar Y_n+\sum_{1\leq i,j\leq n} \frac{1-e^{-\mu_j t}}{\mu_j}\left(v_i^\mathrm{T}u_j u_j^\mathrm{T} \tilde Y\right) g_i(x),
\end{equation}
with $(v_i)_{1\leq i\leq n}$ the canonical basis in $\mathbb{R}^n$. When $\mu=0$, we use extension by continuity, that is the convention $(1-e^{-\mu t})/\mu=t$.
\end{corollary}

Interestingly, the limit function $(\hat F_t)_{t\geq 0}$ appearing in the vanishing learning rate asymptotic can be characterized as the solution of a linear differential equation in  infinite dimensional space. The intuition is quite clear from the following heuristic: the boosting dynamic 
\[
\hat F_{m+1}^\lambda= \hat F_m^\lambda+\lambda \sum_{i=1}^n (Y_i-\hat F_{m}^\lambda(x_i))g_i
\]
implies, for  $t=\lambda m$,  
\[
\lambda^{-1}\left(\hat F_{[(t+\lambda)/\lambda]}^\lambda -\hat F_{[t/\lambda]}^\lambda\right)= \sum_{i=1}^n (Y_i-\hat F_{[t/\lambda]}^\lambda(x_i))g_i.
\]
Letting  $\lambda\to 0$, the convergence $\hat F_{[t/\lambda]}^\lambda\to \hat F_t$  suggests 
\[
 \hat F_t'=\sum_{i=1}^n (Y_i-\hat F_{t}(x_i))g_i.
\]
We make this heuristic rigorous in the following proposition. For $t\geq 0$, we consider  $\hat F_t$ as an element of the Banach space $L^\infty=L^\infty([0,1]^p,\mathbb{R})$ and prove that $(\hat F_t)_{t\geq 0}$ is the unique solution of a linear differential equation. More precisely, it is easily seen that the linear operator  $\mathcal{L}:L^\infty\to L^\infty$
defined by 
\[
\mathcal{L}(Z)=\sum_{i=1}^n Z(x_i)g_i,\quad Z\in L^\infty,
\]
is bounded and we consider the differential equation in the Banach space $L^\infty$ 
\begin{equation}\label{eq:EDO}
Z'(t)=-\mathcal{L}(Z(t))+G,\quad t\geq 0,
\end{equation}
with $G=\sum_{i=1}^n Y_i  g_i$.

\begin{theorem}\label{thm:EDO}
\begin{enumerate}
\item[i)]  For all $Z_0\in L^\infty$, the differential equation \eqref{eq:EDO} has a unique solution satisfying $Z(0)=Z_0$. Furthermore, if there exists  $\mathcal{Y}\in L^\infty$ such that $\mathcal{L}(\mathcal{Y})=G$, this solution is explicitly given  by
\begin{equation}\label{eq:EDO-explicit}
Z(t)=(e^{-t\mathcal{L}})Z_0+(\mathrm{Id}-e^{-t\mathcal{L}})\mathcal{Y},\quad t\geq 0.
\end{equation}
\item[ii)] The function $(\hat F_t)_{t\geq 0}$  is the solution of~\eqref{eq:EDO} with initial condition $\bar Y_n$. Assuming there exists $\mathcal{Y}\in L^\infty$ such that $\mathcal{L}(\mathcal{Y})=G$, we thus have
\[
\hat F_t=(e^{-t\mathcal{L}})\bar Y_n+(\mathrm{Id}-e^{-t\mathcal{L}})\mathcal{Y},\quad t\geq 0.
\]
\end{enumerate}
\end{theorem}

\begin{remark}
The condition $\mathcal{L}(\mathcal{Y})=G$ is satisfied as soon as $\mathcal{Y}(x_i)=Y_i$, $1\leq i\leq n$. In particular, it holds if the $x_i$'s are pairwise distinct. It is used mostly for convenience and elegance of notations. Indeed we have
\[
(\mathrm{Id} -e^{-t\mathcal{L}})(\mathcal{Y})=-\sum_{k\geq 1} \frac{(-t)^k}{k!}\mathcal{L}^k(\mathcal{Y})=\sum_{k\geq 1} \frac{(-1)^{k-1}t^{k}}{k!}\mathcal{L}^{k-1}(G)
\]
and, if the existence of $\mathcal{Y}$ is not granted, one can replace in formula \eqref{eq:EDO-explicit} the term involving $\mathcal{Y}$ by the series in the right hand side of the previous equation  and check that this provides a solution of \eqref{eq:EDO} in the general case.
\end{remark}

Finally, we discuss the notion of stability of the  boosting procedure. It requires that the output of the boosting algorithm does not explodes for large time values.
\begin{definition}
The boosting algorithm is called \emph{stable} if, for all possible input $(Y_i)_{1\leq i\leq n}$, the output $(\hat F_t)_{t\geq 0}$ remains uniformly bounded as $t\to \infty$. 
\end{definition}

It is here convenient to assume the following:
\begin{assumption}\label{ass:cst} In Equation~\eqref{eq:linear-assumption}, the functions $(g_i)_{1\leq i\leq n}$ are linearly independent and such that $\sum_{i=1}^n g_i(x)\equiv 1$.
\end{assumption}
\noindent The linear independence is sensible if the  points $(x_i)_{1\leq i\leq n}$ are pairwise distinct.  The constant sum implies that for constant input $Y_i= 1$, $1\leq i\leq n$, the output $L(x)\equiv 1$ is also constant. Both are  mild assumptions satisfied by most learners in practice.

The stability can be characterized in terms of the Jordan normal form of the matrix $S$, see for instance \cite{HJ13}. We recall that the Jordan normal form of $S$ is a block diagonal matrix where each block, called Jordan block, is an  upper triangular matrix of size $s\times s$ with a complex  eigenvalue $\mu$ on the main diagonal and ones on the superdiagonal. The matrix can be diagonalized if and only if all its Jordan blocks have size $1$. 
\begin{proposition}\label{prop:stability}
Suppose Assumptions~\ref{ass:BY} and~\ref{ass:cst} are satisfied. Then the boosting procedure algorithm is stable if and only if all the blocks of the Jordan normal form of $S$ satisfy: 
\begin{itemize}
\vspace{-0.15cm}
\item[-] the eigenvalue has a positive real part;
\vspace{-0.25cm}
\item[-] the eigenvalue has a null real part and the block has size $1$.
\end{itemize}
In particular, if $S$ is symmetric, the boosting procedure is stable if and only if all the eigenvalues of $S$ are non-negative.
\end{proposition}

\subsection{Training and test error}\label{sec:tr-te}

We next consider the performance of the boosting regression algorithm in terms of $L^2$-loss, also known as mean squared error. We focus mostly on the vanishing learning rate asymptotic, although version of the results below could be derived for positive learning rate $\lambda$.

The training error is assessed on the training set used to fit the boosting predictor and compares  the observations $Y_i$ to their predicted values $\hat F_{t}(X_i)$, i.e. 
\begin{equation}\label{eq:err-train}
\mathrm{err}_{train}(t)=\frac{1}{n}\sum_{i=1}^n(Y_i-\hat F_{t}(x_i))^2.
\end{equation}

The generalization capacity of the algorithm is assessed on new observations that are not used during the fitting procedure. For test observations $(Y_i',X_i')_{1\leq i\leq n'}$, independent of the training sample, the test error is defined by
\begin{align}\label{eq:err-test}
\mathrm{err}_{test}(t)=\frac{1}{n'}\sum_{i=1}^{n'}(Y_i'-\hat F_{t}(X_i'))^2.
\end{align}
We also consider a simpler version of the test error where extrapolation in the feature space is not evaluated and we take $n'=n$ and $X'_i=x_i$. Then, the test error writes
\begin{align}\label{eq:err-test-2}
\mathrm{err}_{test}(t)=\frac{1}{n}\sum_{i=1}^{n}(Y_i'-\hat F_{t}(x_i))^2,
\end{align}
and allows for simpler formulas with nice interpretation.

\medskip
We first consider the behavior of the training error as defined in Equation~\eqref{eq:err-train}.  Note that
\[
\mathrm{err}_{train}(t)=\frac{1}{n}\|R_t\|^2 
\]
where $R_t$ is the vector of residuals at time $t$ defined by
\[
R_t=(Y_i-\hat F_t(x_i))_{1\leq i\leq n},\quad t\geq 0,
\]
and $\|\cdot\|$ denotes the Euclidean norm on $\mathbb{R}^n$. Furthermore, Proposition~\ref{prop:vanishing-learning-rate} implies $R_t=e^{-tS} \tilde Y$, $t\geq 0$,  
so that
\[
\mathrm{err}_{train}(t)=\frac{1}{n}\|e^{-tS} \tilde Y\|^2,\quad t\geq 0. 
\]
The following proposition is related to Proposition~3 and Theorem~1 in \cite{BY03}.
\begin{proposition}\label{prop:train-error} Suppose Assumptions~\ref{ass:BY} and~\ref{ass:cst} are satisfied.
\begin{itemize}
\item[i)]  We have $\lim_{t\to \infty} \mathrm{err}_{train}(t)=0$ for all possible input $(Y_i)_{1\leq i\leq n}$ if and only if all the eigenvalues of $S$ have a positive real part.
\item[ii)] The  training error satisfies
\begin{align}
\mathbb{E}[\mathrm{err}_{train}(t)] &=\mathrm{bias}^2(t)+\mathrm{var}_{train}(t),\label{eq:expected-training-error}\\
\mathrm{bias}^2(t)&=\frac{1}{n}\|e^{-tS}\tilde f\|^2,\nonumber\\
\mathrm{var}_{train}(t)&=\frac{\sigma^2}{n}\mathrm{Trace}\left( e^{-tS}Je^{-tS^T}\right),\nonumber
\end{align}
with $J=I-\frac{1}{n}1_n1_n^T$, $\tilde f=f-\bar f1_n$, $f=(f(x_i))_{1\leq i\leq n}$ and $\bar f=\frac{1}{n}\sum_{i=1}^n f(x_i)$.
\item[iii)] If $S$ is symmetric  with positive eigenvalues $(\mu_i)_{1\leq i\leq n}$ and corresponding eigenvectors $(u_i)_{1\leq i\leq n}$,  
\[
\mathbb{E}[\mathrm{err}_{train}(t)] =\frac{1}{n}\sum_{i=1}^n (u_i^T\tilde f)^2e^{-2t\mu_i}+ \frac{\sigma^2}{n} \sum_{i=1}^n \|Ju_i\|^2e^{-2t\mu_i}.
\]
The expected training error is strictly decreasing and converges to $0$ exponentially fast as $t\to\infty$.
\end{itemize}
\end{proposition}

The convergence of the training error to zero  implies that the boosting procedure is stable as considered in Proposition~\ref{prop:stability} but the converse is not true since some eigenvalues may have a real part equal to zero. When $S$ is symmetric definite positive, the expected training error converges exponentially fast to $0$ (this was already proved in \cite{BY03} Theorem~1 for $\lambda>0$) but this exponential rate of convergence has to be taken with care since $S$ may have very small eigenvalues, see the numerical illustration in Section~\ref{sec:illustration}.

The fact that the residuals converge to zero suggests that the boosting procedure eventually overfits the training observations and loses generalization power. A simple analysis of this overfit is provided by the test error with fixed covariates $X_i'=x_i$, as defined by Equation~\eqref{eq:err-test-2}. For the sake of simplicity, we emphasize the case when  $S$ is symmetric.

\begin{proposition}\label{prop:test-error-2}
\begin{itemize}
\item[i)] The test error with fixed covariates defined by Equation~\eqref{eq:err-test-2} satisfies
\begin{align*}
\mathbb{E}[{\mathrm{err}}_{test}(t)]&=\mathrm{bias}^2(t)+\mathrm{var}_{test}(t),\\
\mathrm{bias}^2(t)&=\frac{1}{n}\|e^{-t S}\tilde f\|^2, \\
\mathrm{var}_{test}(t)&=\frac{\sigma^2}{n}+\frac{\sigma^2}{n}\mathrm{Trace}\left((I-e^{-tS})J(I-e^{-tS})^T\right).
\end{align*}
\item[ii)] If $S$ is symmetric with positive eigenvalues $(\mu_i)_{1\leq i\leq n}$ and associated eigenvectors $(u_i)_{1\leq i\leq n}$, 
\begin{align*}
\mathrm{bias}^2(t)&=\frac{1}{n}\sum_{i=1}^n (u_i^T\tilde f)^2e^{-2t\mu_i}, \\
\mathrm{var}_{test}(t)&=\sigma^2+\frac{\sigma^2}{n} +\frac{\sigma^2}{n} \sum_{i=1}^n \|Ju_i\|^2(1-e^{-t\mu_i})^2,
\end{align*}
so that the the following properties hold:
\begin{itemize}
\item[-] the squared bias  is decreasing, convex and vanishes as $t\to\infty$;
\item[-] the variance  is increasing  and with limit $2\sigma^2$ as $t\to\infty$;
\item[-] the expected test error is decreasing in the neighborhood of zero, eventually increasing  and with limit $2\sigma^2$ as $t\to\infty$.
\end{itemize}
\end{itemize}
\end{proposition}

We retrieve with explicit theoretical formulas the known behavior of boosting  in practice: the choice of $t\geq 0$ is crucial in the bias/variance trade-off. Small values of $t\geq 0$ lead to underfitting while overfitting appears for larger time values. In the early stage of the procedure, the bias decreases more rapidly that the variance increases, leading to a reduced test error. In practice, cross-validation and early stopping is used to estimate the test error and choose when to stop the boosting procedure, see \cite{ZY05}.

\begin{remark} When the boosting algorithm is initialized at $\hat F_0=0$ as in \cite{BY03}, the expected  training and test error from Propositions~\ref{prop:train-error} and \ref{prop:test-error-2} become
\[
\mathbb{E}[\mathrm{err}_{train}(t)] =\frac{1}{n}\sum_{i=1}^n (u_i^T f)^2e^{-2t\mu_i}+ \frac{\sigma^2}{n} \sum_{i=1}^n \|u_i\|^2e^{-2t\mu_i}
\]
and
\[
\mathbb{E}[\mathrm{err}_{test}(t)] =\frac{1}{n}\sum_{i=1}^n (u_i^T f)^2e^{-2t\mu_i}+\frac{\sigma^2}{n}+ \frac{\sigma^2}{n} \sum_{i=1}^n \|u_i\|^2e^{-2t\mu_i}.
\]
These values are always larger than those with initialization  $\hat F_0=\bar Y_n$, whence we recommend initialization to the empirical mean. 
\end{remark}

When the test error includes extrapolation in the predictor space - i.e. the new test observations $(Y_i',X_i')_{1\leq i\leq n'}$ are i.i.d. and independent of the training observation as in Equation~\eqref{eq:err-test} - the formula we obtain for its expectation is more difficult to analyse.

\begin{proposition}\label{prop:test-error}
Assume $S$ is symmetric with positive eigenvalues. The test error defined  by Equation~\eqref{eq:err-test} has expectation 
\begin{align*}
\mathbb{E}[\mathrm{err}_{test}(t)]&=\frac{n+1}{n}\sigma^2+\mathbb{E}\big[ \big(f(X')-\bar f- \tilde f^T S^{-1}\left(I-\mathrm{e}^{-tS}\right)g(X')\big)^2\big]\\
&\quad +\sigma^2\mathbb{E}\left[g(X')^T \left(I-\mathrm{e}^{-tS}\right)S^{-1}JS^{-1}\left(I-\mathrm{e}^{-tS}\right)g(X')\right]
\end{align*}
with $g(X')=(g_i(X'))_{1\leq i\leq n}$.
\end{proposition}

\section{Stochastic gradient boosting}\label{sec:SGB}
Following \cite{F02}, it is common practice to use a stochastic version of the boosting algorithm where subsampling is introduced at each step of the procedure. The package \texttt{gbm} by \cite{R07} uses the subsampling rate equal to $50\%$ by default, meaning that each step involves only a subsample with half of the observations randomly chosen. This subsampling is known to have a regularization effect and we consider in this section the existence of the vanishing learning rate limit for such stochastic boosting algorithms.  

\subsection{Framework}
We consider the following stochastic boosting algorithm that encompasses stochastic gradient boosting, see Example~\ref{ex:SGB} below. We assume the weak learner $L(x)=L(x;(x_i,y_i)_{1\leq i\leq n}, \xi)$ depends on the observations $(x_i,y_i)_{1\leq i\leq n}$ and on an external source of randomness $\xi$ with a finite set   $\Xi$ of possible values.  We define the stochastic boosting algorithm by the recursion
\begin{equation}\label{eq:stochastic-boosting-recursion}
\left\{\begin{array}{ll}
\hat F_0^\lambda(x)&= \bar Y_n, \\
\hat F_{m+1}^\lambda(x)&=\hat F_m^\lambda(x)+\lambda L(x;(R_{m,i}^\lambda,X_i)_{1\leq i\leq n},\xi_{m+1}),\quad m\geq 0,
\end{array}\right.
\end{equation}
where  $\xi_m$, $m\geq 1$, are i.i.d. $\Xi$-valued random variables  independent of $(X_i,Y_i)_{1\leq i\leq n}$ and $R_{m,i}^\lambda=Y_i-\hat F_{m}^\lambda(X_i)$, $1\leq i\leq n$, are the residuals. 

\begin{assumption}\label{ass:randomized-setting}
We assume that the base learner of the stochastic boosting algorithm~\eqref{eq:stochastic-boosting-recursion} satisfies
\[
L(x;(x_i,Y_i)_{1\leq i\leq n},\xi)= \sum_{j=1}^n Y_jg_j(x,\xi),\quad x\in[0,1]^p,
\]
where $g_1,\ldots,g_n\in L^\infty$ may depend on $(x_i)_{1\leq i\leq n}$ and $\xi\in\Xi$.
\end{assumption}

We assume that $\Xi$ is finite mostly for simplicity and also because it is enough to cover two particularly important cases.

\begin{example}\label{ex:SGB} Starting from a base learner $L$ satisfying Assumption~\ref{ass:BY} (with $n$ replaced by $[sn]$) and applying stochastic subsampling \citep{F02}, we obtain a stochastic setting that satisfies Assumption~\ref{ass:randomized-setting}. Let the sample size $n\geq 1$ be fixed and consider subsambling with rate $s\in (0,1)$, e.g. $s=50\%$. Define $\Xi$ as the set of all subsets $\xi$ of $\{1,\ldots,n\}$ with fixed size $[sn]$. Note that $\Xi$ is finite with cardinality $n\choose{[sn]}$. The learner $L$ fitted on subsample $\xi\in\Xi$ is written
\[
L(x;(x_i,Y_i)_{1\leq i\leq n},\xi)=L(x;(x_i,Y_i)_{i\in \xi}).
\]
We use here a mild abuse of notation: in the left hand side, $L$ denotes the randomized learner, the sample size is $n$ and subsampling is introduced by $\xi$; in the right hand side, $L$ denotes the deterministic base learner and the sample size is $[sn]$. Stochastic boosting corresponds to Algorithm~\eqref{eq:stochastic-boosting-recursion} with the sequence $(\xi_m)_{m\geq 1}$ uniformly distributed on $\Xi$, which corresponds to uniform subsampling.
\end{example}

\begin{example} Another important example covered by the stochastic boosting algorithm~\eqref{eq:stochastic-boosting-recursion} is the design of additive models. The idea is to provide an approximation of the regression function $f(x)$, $x\in[0,1]^p$, by an additive model of the form $f_1(x^{(1)})+\cdots+f_p(x^{(p)})$, where $x^{(j)}$ denotes the $j$th component of $x$ and $f_j$ the principal effect of $x^{(j)}$. Such an additive model does not include interactions between different components.  Assume that a base learner $L$ with one-dimensional covariate space $[0,1]$ is given and that $L$ satisfies Assumption~\ref{ass:BY} with $p=1$. For instance, $L$ can be a smoothing spline as in Example~\ref{ex:smoothing-spline}, see \cite{BY03} Section~4. We consider stochastic regression boosting where the base learner $L$ is sequentially applied with a randomly chosen predictor. Formally, set 
\[
L(x;(x_i,Y_i)_{1\leq i\leq n},\xi)=L(x;(x_i^{(\xi)},Y_i)_{1\leq i\leq n}),\quad \xi=1,\ldots,p.
\]
It is easily checked that the learner in the left hand side satisfies Assumption~\ref{ass:randomized-setting} and that  algorithm~\eqref{eq:stochastic-boosting-recursion} with  $(\xi_m)_{m\geq 1}$ uniformly distributed on $\Xi=\{1,\ldots,p\}$ outputs a sequence of additive models. This strategy is often used with a more involved procedure  where, at each step, the $p$ different possible predictors are considered and the best one is kept, see \cite{BY03} Section~4. But this falls beyond Assumption~\ref{ass:randomized-setting} because choosing the optimal component is not a linear operation and the randomized choice proposed here is a sensible alternative satisfying Assumption~\ref{ass:randomized-setting}.
\end{example}

\subsection{Convergence of finite dimensional distributions} 

For fixed input $(Y_i,x_i)_{1\leq i\leq n}$, the stochastic boosting algorithm \eqref{eq:stochastic-boosting-recursion} provides a sequence of stochastic processes $\hat F_m^\lambda$, $m\geq 1$, and we consider the  vanishing learning rate limit \eqref{eq:limit} under Assumption~\ref{ass:randomized-setting}. We first prove  convergence of the finite dimensional distributions thanks to  elementary moment computations formulated in the next proposition. Expectation and variance are considered with respect to $(\xi_m)_{m\geq 1}$ while the input $(x_i,Y_i)_{1\leq i\leq n}$ is considered fixed and we note $\mathbb{E}_\xi$ and $\mathrm{Var}_\xi$ to emphasize this. We define
\[
\bar g_j(x)=\mathbb{E}_\xi[g_j(x,\xi)],\quad x\in[0,1]^p,\ j=1,\ldots,n,
\]
and
\begin{equation}\label{eq:S}
S=(\bar g_i(x_j))_{1\leq i,j\leq n}.
\end{equation}
Note that $\bar g_1,\ldots,\bar g_n$ are well-defined and in $L^\infty$ because  $\Xi$ is finite so that there are no measurability or integrability issues. 

\begin{proposition}\label{prop:randomized-setting-fdd}
Consider the boosting algorithm \eqref{eq:stochastic-boosting-recursion} under Assumption~\ref{ass:randomized-setting} and let the input $(Y_i,x_i)_{1\leq i\leq n}$ be fixed. 
\begin{enumerate}
\item[i)] For $x\in [0,1]^p$ and $m\geq 0$,
\begin{equation}\label{eq:expectation}
\mathbb{E}_\xi[\hat F_m^\lambda(x)]=\bar Y_n +\sum_{i=1}^n w_{m,i}^\lambda \bar g_i(x),
\end{equation}
where  $w_{m}^\lambda=(w_{m,i}^\lambda)_{1\leq i\leq n}$ is defined by~\eqref{eq:weight} with $S$ given by \eqref{eq:S}. 
\item[ii)] There exists a positive constant $K$ such that, for all $x\in [0,1]^p$, $m\geq 0$ and $\lambda<1$,  
\[
\mathrm{Var}_\xi[\hat F_m^\lambda(x)]\leq K(m+1)\lambda^2 (1+K\lambda)^mn\|\tilde Y\|_\infty^2 \left\lbrace 1+ (\lambda m K)^2\mathrm{e}^{2\lambda m \|S\|_\infty}\right\rbrace,
\]
where $\|\cdot\|_\infty$ denotes here the maximum norm on $\mathbb{R}^n$. We use the same notation for the infinity-norm of $n\times n$ matrices.
\end{enumerate}
\end{proposition}
As will be clear from the proof, the constant $K$ can be taken as $2M_1+M_1^2+(n+1)M_2$, where 
\begin{equation}\label{eq:M_1}
M_1=\max_{1\leq j\leq n+1}\sum_{i=1}^n|\bar g_i(x_j)|
\end{equation}
and 
\begin{equation}\label{eq:M_2}
M_2=\max_{1\leq j\leq n+1}\sum_{i=1}^n\mathrm{Var}_\xi[g_i(x_j)].
\end{equation}

\begin{corollary}\label{cor:vanishing-limit-stochastic}
For all $t\geq 0$ and $x\in[0,1]^p$,  the convergence~\eqref{eq:limit} holds in quadratic mean with  deterministic limit
\begin{equation}\label{eq:def-F_t-limit}
\hat F_t(x)=\bar Y_n+\sum_{i=1}^n w_{t,i} \bar g_i(x),
\end{equation}
where $w_{t}=(w_{t,i})_{1\leq i\leq n}$ is defined by~\eqref{eq:wt} with $S$ given by \eqref{eq:S}. 
\end{corollary}

Corollary~\ref{cor:vanishing-limit-stochastic} states that the vanishing learning rate limit $\hat F_t(x)$ for the stochastic boosting algorithm \eqref{eq:stochastic-boosting-recursion} is exactly the same as the one for the deterministic boosting algorithm \eqref{eq:boosting-recursion} with base learner
\begin{align*}
\bar L(x;(x_i,Y_i)_{1\leq i\leq n})&=\mathbb{E}_\xi [L(x;(x_i,Y_i)_{1\leq i\leq n},\xi)]\\
&=\sum_{j=1}^n Y_j\bar g_j(x).
\end{align*}
In particular, the properties of the limit process $(\hat F_t)_{t\geq 0}$ have been studied in Section~\ref{sec:L2boosting-BY}: characterization by a differential equation, behaviour of the training and test error, \textit{etc}.

\subsection{Weak convergence in function space}
Corollary~\ref{cor:vanishing-limit-stochastic} implies that the convergence 
\[
\hat F_{[t/\lambda]}^\lambda(x)\longrightarrow \hat F_t(x), \quad \mbox{as $\lambda\to 0$},
\]
holds in the sense of finite dimensional distributions, that is joint convergence in distribution of the values of the processes at finitely many points $(t_i,x_i)_{1\leq i\leq k}$.  We strengthen here this convergence into a  weak convergence of stochastic processes in the Skorokhod space $\mathbb{D}([0,\infty),L^\infty)$ of  c\`ad-l\`ag functions with values in $L^\infty$. We refer to \citet[Section 16]{B99} for background on the Skorokhod space $\mathbb{D}([0,\infty),\mathbb{R})$ and to \citet[Section 3.5]{EK86} for the general Skorokhod space $\mathbb{D}([0,\infty),E)$ on a metric space $E$. 

\begin{theorem}\label{thm:Markov-functional}
Consider the boosting algorithm \eqref{eq:stochastic-boosting-recursion} under Assumption~\ref{ass:randomized-setting} and let the input $(Y_i,x_i)_{1\leq i\leq n}$ be fixed.
\begin{enumerate}
\item[i)] For fixed $\lambda>0$, the boosting sequence $(\hat F_m^\lambda)_{m\geq 0}$ is a time homogeneous Markov chain with values in $L^\infty$.
\item[ii)] As $\lambda\to 0$, the convergence in distribution
\begin{equation}\label{eq:conv-functional}
(\hat F_{[t/\lambda]}^{\lambda})_{t\geq 0}\stackrel{d}\longrightarrow (\hat F_{t})_{t\geq 0}
\end{equation}
holds in the Skorokhod space $\mathbb{D}([0,\infty),L^\infty)$.
\end{enumerate} 
\end{theorem}
 Our proof relies on the theory of approximations of Markov chains by diffusions, see e.g. \cite{SV06}. The assumption that $\Xi$ is finite is important because it ensures that  $(\hat F_m^\lambda)_{m\geq 0}$ remains in the finite dimensional subspace 
\[
\mathcal{F}=\mathrm{span}(1,g_i(\cdot,\xi); 1\leq i\leq n, \xi\in\Xi),
\]
where $\mathrm{span}$ denotes the linear span of vectors in $L^\infty$. In Theorem~\ref{thm:Markov-functional}, the Markov property and the functional convergence can equivalently be stated with $L^\infty$ replaced by the finite-dimensional space $\mathcal{F}$. This property is crucial in order to use the theory of multidimensional diffusion by \cite{SV06}. 

The limit process $(\hat F_{t})_{t\geq 0}$ is not only c\`ad-l\`ag but continuous  with respect to time as  it is the solution of the linear differential Equation~\eqref{eq:EDO}, where the functions $(g_i)_{1\leq i\leq n}$ need to be replaced by $(\bar g_i)_{1\leq i\leq n}$ in the definition of $\mathcal{L}$. The limit process is even smooth in time as shown by the explicit solution given in Theorem~\ref{thm:EDO}.

\section{Numerical illustration}\label{sec:illustration}
We study numerically the behavior of the vanishing learning rate limit $\hat{F}_t(\cdot)$ given in Equation~\eqref{eq:Ft}. We consider the experimental design from \cite{ZY05} Section 6.1: the  sample  $(Y_i,X_i)_{1\leq i\leq n}$ is generated according to 
\begin{equation}\label{sim-model}
\left\{ \begin{array}{l}
X_i\sim \mathrm{Unif}([-1,1]),\\
\varepsilon_i\sim \mathcal{N}(0,1/4),\\
Y_i=f(X_i)+\varepsilon_i,
\end{array}\right.
\end{equation} 
with regression function
\[
f(x)=1-\Big|2|x|-1\Big|,\quad x\in [-1,1].
\]
The covariates $(X_i)_{1\leq i\leq n}$ and the errors $(\varepsilon_i)_{1\leq i\leq n}$ are assumed i.i.d. and independent on each other. For the boosting procedure, we use a cubic smoothing spline   with 5 degrees of freedom as linear base learners, see Example \ref{ex:smoothing-spline} with $r=2$. Recall that the degrees of freedom, noted $\mathrm{df}$, is equal to the trace of the matrix $S$ defined in Equation \eqref{weights-vec} and reflects the complexity of the base learner.

 We first simulate $n=100$ observations of model \eqref{sim-model} and compute the limit functions $\hat{F}_t(\cdot)$ for different time values $t=0$, $1$, $10$, $100$ and $1\,000$. We can see that $t=10$ produces a fairly good fit while $t=0$ or $1$ produces an underfit and $t=100$ or $1000$ an overfit. Such large values of $t$ are rarely used in practice and this shows that overfitting eventually arises, but very slowly. 

\begin{figure}[H]
\begin{center}\includegraphics[width=12cm,height=9cm]{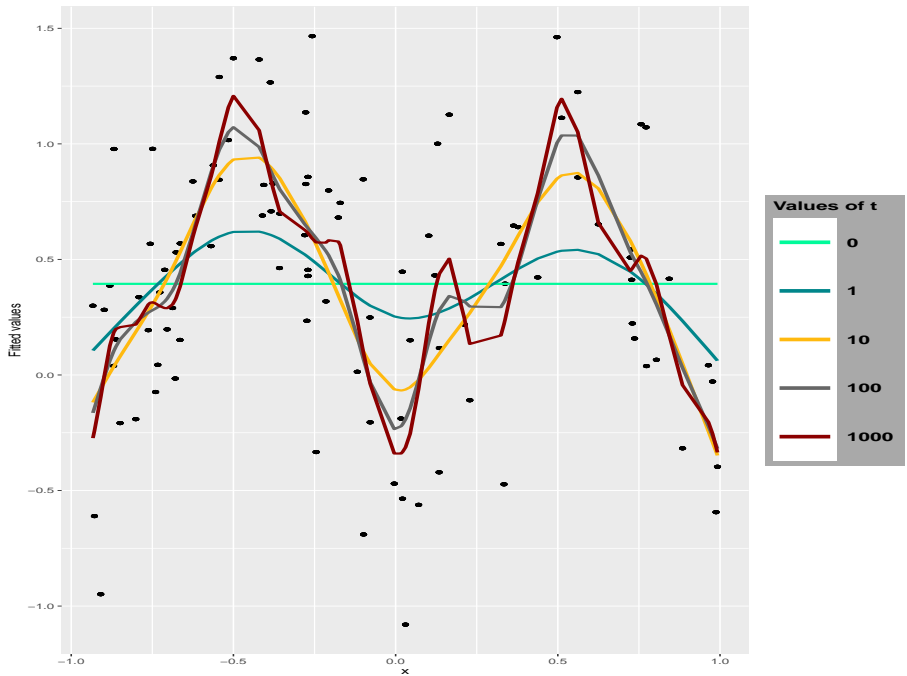}
\caption{Output of the L$^2$-Boosting algorithm in the vanishing learning rate asymptotic with input $(Y_i,X_i)_{1\leq i\leq 100}$ generated from model \eqref{sim-model} at different values of $t$. The black dots represent the observed data set $(Y_i,X_i)_{1\leq i\leq 100}$.}
\label{fig1}
\end{center}
\end{figure}

To analyse this overfit and the effect of the learning rate $\lambda$, we then compare the training and test errors as defined in Section \ref{sec:tr-te}. In Figure \ref{fig2} below, we plot the training and test errors of the boosting predictor $\hat{F}^\lambda_{[t/\lambda]}(\cdot)$ as a function of time $t\geq 0$ (in logarithmic scale), for different learning rates $\lambda=1$, $0.5$, $0.1$ and for the vanishing learning rate limit $\lambda\to 0$. We can see that the training error is decreasing while the test error decreases until a minimum at $t\approx \exp(1.8)\approx 6$ and starts to increase again. As $\lambda\to 0$,  the error functions converge quickly to their limit and $\lambda=0.1$ can hardly be distinguished from the  limit.  Furthermore, convergence is slower for small time values and uniformly fast for large time values. Interestingly, the test error is minimal for vanishing learning rate  $\lambda\to 0$, especially for small time values, supporting the idea that reducing the learning rate reduces the test error.

\begin{figure}[H]
\begin{center}\includegraphics[width=13cm,height=9cm]{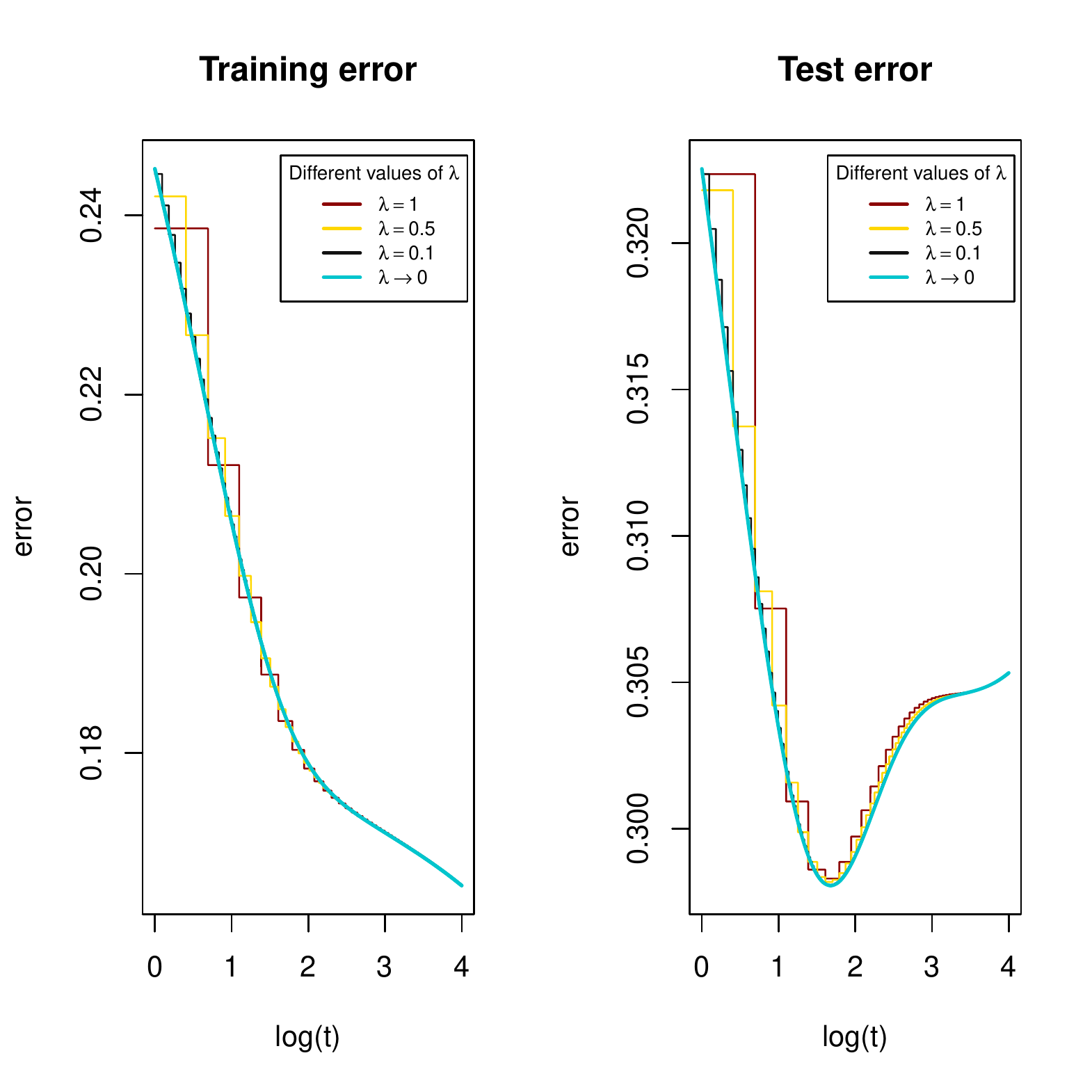}
\caption{Training and test error of the predictors $\hat{F}_t(\cdot)$ and $\hat{F}^\lambda_{[t/\lambda]}(\cdot)$ for different values of $\lambda$.}
\label{fig2}
\end{center}
\end{figure}

Importantly,  training error decreases slowly than expected, in contrast with the exponential rate of convergence announced by the theory. The convergence to $0$ of the training error cannot be observed on Figure~\ref{fig2} where $t$ ranges from $0$ to $\exp(4)\approx 50$.  This phenomenon is explained by the order of magnitude of the eigenvalues of the base learner. The eigenvalues of  $S$ quickly decrease  to very small values as seen in Figure \ref{fig3} below, where the $60$ largest eigenvalues are plotted in logarithmic scale (some numerical instability arises for smaller eigenvalues). The rate of convergence to $0$ of the training error is exponential with rate  $e^{-t\mu_{100}}$, where $\mu_{100}$ denotes the smallest eigenvalues. Here we already have $\mu_{60}\approx 6\cdot 10^{-7}$ explaining the slow rate of decrease of the training error and the fact tjat convergence to zero  is not  observed in practice on usual time range.

\begin{figure}[H]
\begin{center}\includegraphics[width=10cm,height=8cm]{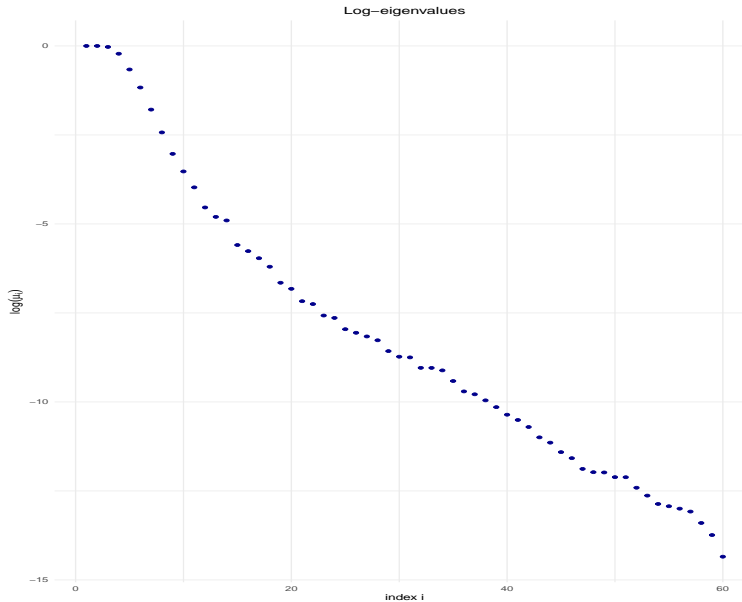}
\caption{Decay of the base learner eigenvalues in logarithmic scale - only the $60$ largest eigenvalues are plotted due to numerical instability for smaller ones.}
\label{fig3}
\end{center}
\end{figure}  

We next discuss what this behaviour of eigenvalues imply for the boosting predictor $\hat F_t(\cdot)$. When considering prediction at $\mathbf{x}=(x_i)_{1\leq i\leq n}$, we can write $\hat F_t(\mathbf{x})=(\hat F_t(x_i))_{1\leq i\leq n}$  as
\begin{equation}\label{eq:smoothed-projection}
\hat F_t(\mathbf{x})=u_1u_1^TY+\sum_{i=2}^n (1-e^{-\mu_i t})u_iu_i^TY
\end{equation}
where  $(\mu_i)_{1\leq i\leq n}$ and  $(u_i)_{1\leq i \leq n}$ are the eigenvalues and eigenvectors of $S$ and $Y=(Y_i)_{1\leq i\leq n}$. The rank one matrix $u_iu_i^T$ is the matrix of the orthogonal projection on the eigenspace $\mathbb{R}u_i$.  We interpret Equation~\eqref{eq:smoothed-projection} as a smoothed projection: for $\mu_it\gg 1$, $1-e^{-\mu_i t}\approx 1$ and the boosting operator acts as the projection on this dimension; on the opposite, for $\mu_it\ll 1$, $1-e^{-\mu_i t}\approx 0$ and the boosting operator acts as a filter on this dimension. The fact that the eigenvalues are spread out on multiple orders of magnitude implies that most of the coefficients are close to $0$ or $1$, whence the name smoothed projection. This is illustrated on Figure~\ref{fig4} where the coefficients from Equation~\eqref{eq:smoothed-projection} are plotted for various values of $t$. The sum of the coefficient is equal to the trace of the linear boosting operator, that is  its number of degrees of freedom
\[
\mathrm{df}(t)=1+\sum_{i=2}^n (1-e^{-\mu_i t}),
\]
which is an increasing function of time. As we can see, boosting acts as a smoothed projection on the linear spaced spanned by the first eigenvectors. The larger the time, the larger the number of degrees of freedom of the smoothed projection, that is the larger the projection space.

\begin{figure}[H]
\begin{center}\includegraphics[width=13cm,height=10cm]{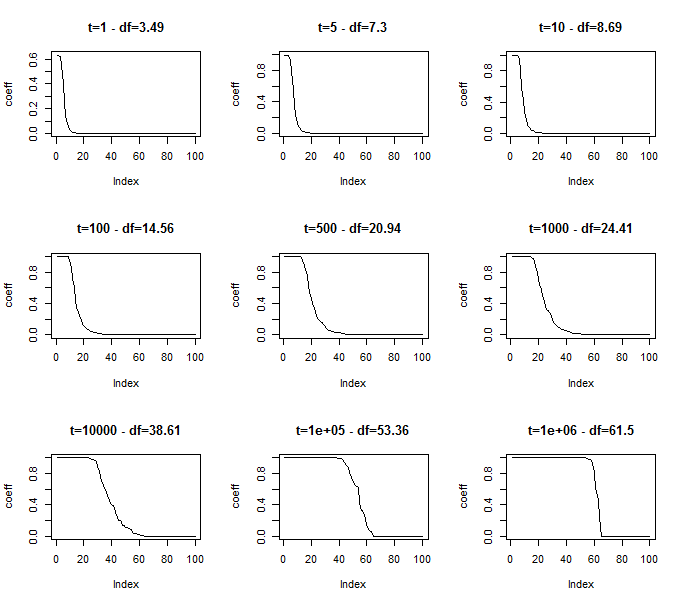}
\caption{Coefficient appearing in Equation~\eqref{eq:smoothed-projection},  the interpretation of the boosting linear operator as a smoothed projection.}
\label{fig4}
\end{center}
\end{figure} 

We finally investigate the behaviour of the eigenvectors of  base learner. Figures~\ref{fig5} and~\ref{fig6} respectively show  the first and last eigenvectors. Quite strikingly, we can see that the first eigenvectors contains a well structured signal (akin to a polynomial basis) while the last eigenvectors mostly contains noise. The interpretation of linear boosting as a smoothed projection is thus meaningful as projection is performed on small dimensions containing signal, while higher dimensions associated to noise are filtered out.

\begin{figure}[H]
\begin{center}\includegraphics[width=10cm,height=8.5cm]{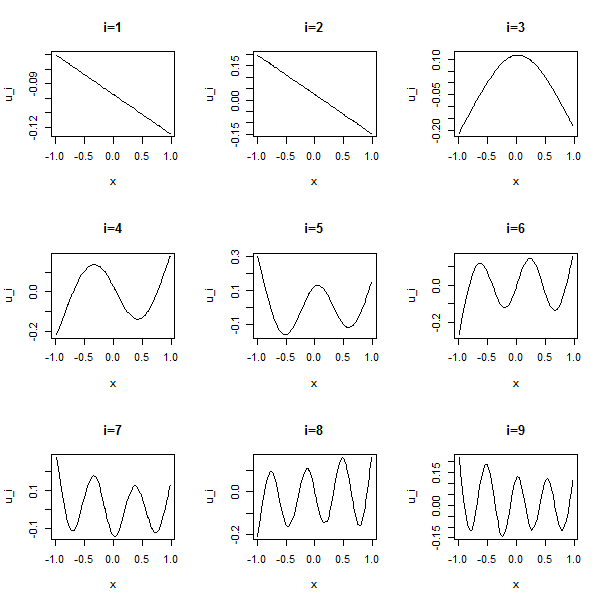}
\caption{First eigenvectors of the linear base learner interpreted as signal.}
\label{fig5}
\end{center}
\end{figure}

\begin{figure}[H]
\begin{center}\includegraphics[width=10cm,height=8.5cm]{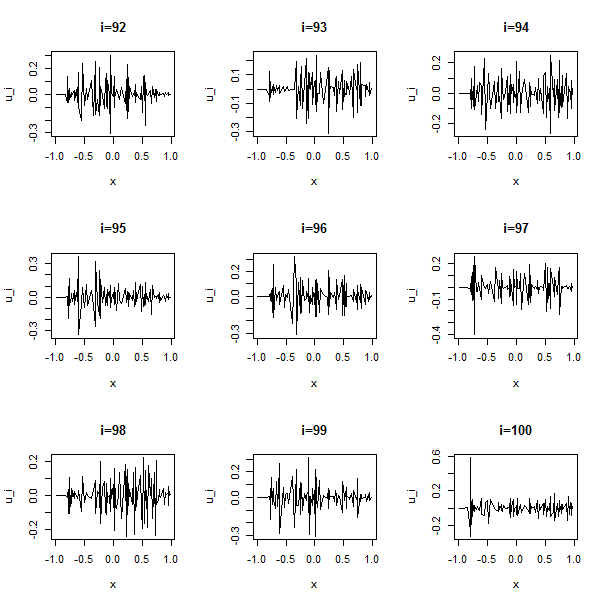}
\caption{Last eigenvectors of the linear base learner interpreted as noise.}
\label{fig6}
\end{center}
\end{figure}

\section{Proofs}\label{sec:proofs}
\subsection{Proofs for Section~\ref{sec:L2boosting-BY}}

\begin{proof}[Proof of Proposition~\ref{prop:weight-formula}]
A straightforward induction based on the recursive Equation~\eqref{weights-vec} yields the first formula for $w_{m}^\lambda$. Furthermore, the identity
\[
(I-\lambda S)\sum_{j=0}^{m-1} (I-\lambda S)^j=\sum_{j=0}^{m} (I-\lambda S)^j -I
\]
implies
\[
\lambda S\sum_{j=0}^{m-1} (I-\lambda S)^j=I-(I-\lambda S)^m.
\]
When $S$ is invertible, we deduce
\[
\lambda \sum_{j=0}^{m-1} (I-\lambda S)^j=S^{-1}[I-(I-\lambda S)^m]
\]
and the second formula for $w_m^\lambda$ follows.
\end{proof}

\begin{proof}[Proof of Proposition~\ref{prop:vanishing-learning-rate}]
Using the first formula for $w_m^\lambda$ from Proposition~\ref{prop:weight-formula} and the binomial theorem, we get
\[
w_{m}^\lambda=\lambda\sum_{j=0}^{m-1}\sum_{k=0}^j\binom{j}{k}\left(-\lambda S\right)^k\tilde Y=\lambda\sum_{k=0}^{m-1}\left(-\lambda S\right)^k\sum_{j=k}^{m-1}\binom{j}{k}\tilde Y.
\]
By the Hockey-stick identity, 
\begin{equation}\label{eq:weight-formula}
w_{m}^\lambda=\lambda\sum_{k=0}^{m-1}\binom{m}{k+1}\left(-\lambda S\right)^k\tilde Y=-\sum_{j=1}^{m}\binom{m}{j}(-\lambda)^j S^{j-1}\tilde Y.
\end{equation}
Let $T>0$. In view of \eqref{eq:linear-formula} and \eqref{eq:Ft}, 
\[
\sup_{t\in[0,T]}\sup_{x\in [0,1]^p}\left|\hat F_{\left[t/\lambda\right]}^\lambda(x)-\hat F_t(x)\right|
\leq \sup_{t\in[0,T]}\sum_{i=1}^n\left|w_{\left[t/\lambda\right],i}^\lambda-w_{t,i}\right|\sup_{x\in [0,1]^p}\left|g_i(x)\right|.
\]
The functions $g_1,\ldots,g_n$ are locally bounded so that  
\begin{equation}\label{eq:conv-unif-F}
\sup_{t\in[0,T]}\sup_{x\in [0,1]^p}\left|\hat F_{\left[t/\lambda\right]}^\lambda(x)-\hat F_t(x)\right|\leq  M \sup_{t\in[0,T]}\left\|w_{\left[t/\lambda\right]}^\lambda-w_{t}\right\|
\end{equation}
for some $M>0$. Here, $\|\cdot\|$ denotes any norm on $\mathbb{R}^n$ and we use below the same notation for the induced operator norm of $n\times n$ matrix.
We need to prove that 
\begin{equation}\label{eq:conv-w}
\|w_{[t/\lambda]}^\lambda- w_t\|\to 0 \quad \mbox{  uniformly  for $t\in [0,T]$.}
\end{equation}
Equation~\eqref{eq:weight-formula} implies
\[
w_{\left[t/\lambda\right]}^\lambda-w_{t}= \sum_{j=1}^{\left[t/\lambda\right]}\left(\frac{(-t)^j}{j!}-\binom{\left[t/\lambda\right]}{j}(-\lambda)^j\right)S^{j-1}\tilde Y+\sum_{j>\left[t/\lambda\right]}\frac{(-t)^j}{j!}S^{j-1} \tilde Y,
\]
whence we deduce $\|w_{\left[t/\lambda\right]}^\lambda-w_{t}\|\leq \mathrm{I}+\mathrm{II}$
with 
\begin{align*}
\mathrm{I}&=\sum_{j=1}^{\left[t/\lambda\right]} \left|\frac{t^j}{j!}-\binom{\left[t/\lambda\right]}{j}\lambda^j\right|\left\|S\right\|^{j-1}\|\tilde{Y}\|,\\
\mathrm{II}&=\sum_{j>\left[t/\lambda\right]}\frac{t^j}{j!}\|S\|^{j-1} \|\tilde Y\|.
\end{align*}
Consider the first term. For $t\in [0,T]$, $\lambda>0$ and  $1\leq j \leq [t/\lambda]$, 
\[
\binom{\left[t/\lambda\right]}{j}\lambda^j\geq\frac{1}{j!}\left(\lambda\left[t/\lambda\right]-\lambda (j-1)\right)^j\geq \frac{1}{j!}(t-\lambda j)^j,
\]
which entails
\[
\left|\frac{t^j}{j!}-\binom{\left[t/\lambda\right]}{j}\lambda^j\right|\leq \frac{1}{j!}(t^j -(t-\lambda j)^j)\leq \frac{1}{j!}(T^j -(T-\lambda j))^j
\]
where the last inequality uses the fact that $t\mapsto t^j -(t-\lambda j)^j$ is increasing on $[\lambda j,+\infty)$. We deduce 
\[
\mathrm{I}\leq \|\tilde{Y}\|\sum_{j=1}^{\left[t/\lambda\right]} \frac{1}{j!}( T^j-(T-\lambda j)^j)\|S\|^{j-1}.
\]
Thanks to the inequality
\[
T^j-(T-\lambda j)^j= \lambda j\sum_{k=0}^{j-1}T^{j-k-1}(T-\lambda j)^k \leq \lambda j^2 T^{j-1},
\]
we get the upper bound
\[
\mathrm{I}\leq \lambda \|\tilde{Y}\| \sum_{j=1}^{\infty} \frac{1}{j!}j^2(T \|S\|)^{j-1}.
\]
Since the last series converges and $t\in [0,T]$ is arbitrary, we deduce $\mathrm{I}=O(\lambda)$ as $\lambda\to 0$,  uniformly for $t\in [0,T]$.

To analyze the second term, we distinguish between the case $t<\varepsilon$ and $t\geq \varepsilon$. 
For $t\in (0,\varepsilon]$, 
\[
\mathrm{II}\leq \|\tilde Y\| \sum_{j\geq 1}\frac{\varepsilon ^{j}\| S\|^{j-1}}{j!}
\]
and this can be made arbitrary small if we choose $\varepsilon$ small enough (independently of $\lambda$. On the other hand, for $t\in [\varepsilon, T]$,
\[
\mathrm{II}\leq \|\tilde Y\| \sum_{j\geq\left[\delta/\lambda\right]}\frac{T^{j}\| S\|^{j-1}}{j!}
\]
and the right hand side converges to $0$ as $\lambda\to 0$, because it is the remainder of a convergent series. We deduce that $\mathrm{II}\to 0$ as $\lambda\to 0$ uniformly in $t\in [0,T]$.
This proves Equation~\eqref{eq:conv-w} and, 
in view of Equation~\eqref{eq:conv-unif-F},
 the  convergence $\hat F_{[t/\lambda]}^\lambda (x)\to\hat F_t(x)$ uniformly on $[0,T]\times [0,1]^p$.

Finally, in the case when $S$ is invertible, we have
\begin{align*}
w_t&=-\sum_{j\geq 1}\frac{(-t)^j}{j!}S^{j-1} \tilde Y=-S^{-1}\sum_{j\geq 1}\frac{(-t)^j}{j!}S^{j}\tilde Y\\
&=-S^{-1}(e^{-tS}-I) \tilde Y.
\end{align*}
This proves Equation~\eqref{eq:wt-inv}.
\end{proof}

\begin{proof}[Proof of Corollary~\ref{cor:symmetric}]
When $S$ is symmetric, it can be written as the sum of rank $1$ orthogonal projections
\[
S=\sum_{j=1}^n \mu_j u_j u_j^{\mathrm{T}}.
\]
Then we have
\begin{align*}
\sum_{k\geq 1} \frac{(-t)^k}{k!}S^{k-1}&=\sum_{j=1}^n \left(\sum_{k\geq 1} \frac{(-t)^k}{k!}\mu_j^{k-1}\right)u_j u_j^{\mathrm{T}}\\
&= \sum_{j=1}^n \frac{e^{-\mu_j t}-1}{\mu_j} u_j u_j^{\mathrm{T}}
\end{align*}
where the series are normally convergent and extension by continuity is used in the last equality when $\mu_j=0$. We deduce from Proposition~\ref{prop:vanishing-learning-rate} that
\[
w_t=\sum_{j=1}^n \frac{1-e^{-\mu_j t}}{\mu_j}u_j u_j^{\mathrm{T}}\tilde Y
\]
and Equation~\eqref{F-exp-limit} follows from Equation~\eqref{eq:Ft} since $w_{t,i}=v_i^{\mathrm{T}}w_t$, $1\leq i\leq n$.
\end{proof}

\begin{proof}[Proof of Theorem~\ref{thm:EDO}]

We first check the claim that $\mathcal{L}$ is a bounded linear operator on $L^\infty$. We denote by $\|\cdot\|_\infty$ the norm in $L^\infty$. For all $x\in [0,1]^p$,
\[
|\mathcal{L}(Z)(x)|\leq \sum_{i=1}^n |Z(x_i)| |g_i(x)| \leq   \Big\|\sum_{i=1}^n |g_i|\Big\|_\infty \| Z\|_\infty,
\]
whence we deduce, taking the supremum over $x\in [0,1]^p$,
\[
\|\mathcal{L}(Z)\|_{\infty}\leq   \Big\|\sum_{i=1}^n |g_i|\Big\|_\infty \| Z\|_\infty.
\]
This proves that the linear operator $\mathcal{L}$ is bounded.

\smallskip
\noindent \textit{Point $i)$}
In the Banach space $L^\infty$, the differential equation~\eqref{eq:EDO} is linear of first order with constant bounded linear operator $\mathcal{L}$ and hence it follows from the general theory (see \cite{Tom1969}) that it admits a unique solution starting from any point $Z_0$. 
We check that  $Z(t)$ defined by Equation~\eqref{eq:EDO-explicit} is this solution. For $t=0$, $e^{-t\mathcal{L}}=\mathrm{Id}$ so that Equation~\eqref{eq:EDO-explicit} yields $Z(0)=Z_0$. On the other hand, differentiating Equation~\eqref{eq:EDO-explicit} thanks to the relations $(e^{-t\mathcal{L}})'=-\mathcal{L}e^{-t\mathcal{L}}$, we obtain 
\begin{align*}
Z'(t)&=-\mathcal{L}e^{-t\mathcal{L}}Z(0)+ \mathcal{L}e^{-t\mathcal{L}}\mathcal{Y}\\
&=-\mathcal{L}(Z(t) - (\mathrm{Id}-e^{-t\mathcal{L}})\mathcal{Y})+  \mathcal{L}e^{-t\mathcal{L}}\mathcal{Y}\\
&=-\mathcal{L}Z(t)+\mathcal{L}\mathcal{Y}\\
&=-\mathcal{L}Z(t)+G.
\end{align*}
In the last equality, we use the fact that $\mathcal{Y}$ is such that $\mathcal{L}(\mathcal{Y})=G$. This proves that $Z(t)$ is the solution of \eqref{eq:EDO} with initial condition $Z(0)=Z_0$.

\smallskip
\noindent \textit{Point $ii)$}
We finally check that $(\hat F_t)_{t\geq 0}$ is the solution of \eqref{eq:EDO} with initial condition $\bar Y_n$. By construction, we have $\hat F_0=\bar Y_n$. Furthermore, differentiating the relation
\[
\hat F_t=\bar Y_n+\sum_{i=1}^n w_{t,i} g_i
\]
yields
\[
\hat F_t'=\sum_{i=1}^n w_{t,i}' g_i
\]
where $w_{t,i}'$ is the $i$-th component of
\[
w_t'=\sum_{j\geq 1} \frac{(-t)^{j-1}}{(j-1)!}S^{j-1}\tilde Y=e^{-tS}\tilde Y.
\]
The derivative of $w_t$ is obtained by differenting  the power series defining $w_t$ and we can see that $w_t$ satisfies the differential equation
\[
w_t'=-S w_t+\tilde Y,\quad t\geq 0.
\]
As a consequence, for $t\geq 0$, 
\begin{align*}
\hat F_t'&=\sum_{i=1}^n \Big( -\sum_{j=1}^n S_{i,j}w_{t,j}+\tilde Y_i\Big) g_i\\
&= -\sum_{i=1}^n \Big(\bar Y_n  +\sum_{j=1}^n w_{t,j}g_j(x_i)\Big)g_i+  \sum_{i=1}^n Y_ig_i\\
&=-\sum_{i=1}^n \hat F_t(x_i)g_i+  \sum_{i=1}^n Y_ig_i\\
&= -\mathcal{L}(\hat F_t)+G. 
\end{align*}
This proves that $(\hat F_t)_{t\geq 0}$ is the unique solution of Equation~\eqref{eq:EDO} starting from $\bar Y_n$ and its explicit form follows from point $ii)$.
\end{proof}

\begin{proof}[Proof of Proposition~\ref{prop:stability}] 
Using the linear independence of $g_1,\ldots,g_n$ together with Equation~\eqref{eq:Ft}, we can see that the output $\hat F_t$ remains bounded in $L^\infty$ as $t\to\infty$ if and only if the weight $w_t$ remains bounded in $\mathbb{R}^n$ as $t\to\infty$. Using the explicit formula~\eqref{eq:wt} for $w_t$ and the Jordan decomposition of $S$, we prove that $(w_t)_{t\geq 0}$ remains bounded if and only if, for all Jordan block $B$ of $S$,
\begin{equation}\label{eq:series}
\sum_{k\geq 1}\frac{(-t)^k}{k!}B^{k-1} \quad \mbox{remains bounded as $t\to\infty$}.
\end{equation}
Indeed,  the assumption  $\sum_{i=1}^n g_i(x)=1$ implies that $1$ is an eigenvalue of $S$ associated to the constant eigenvector $1_n$. It follows that the centered input $\tilde Y$ in the definition of \eqref{eq:wt} can provide a contribution related to any other Jordan block.

Finally, we characterize the property \eqref{eq:series}. Write the Jordan block $B$ of size $s$  in the form $B=\mu I_s+N$ where $\mu$ is an eigenvalue of $S$ and $N$ a nilpotent matrix of order $s-1$. A standard discussion, akin to the criterion for the stability of linear systems of differential equations (see \cite{bellman1969}), reveals that  \eqref{eq:series} holds if and only if $\mu$ has positive real part or if $s=1$ and $\mu$ has a null real part.
\end{proof}

\begin{proof}[Proof of Proposition~\ref{prop:train-error}] 
\noindent \textit{Point $i)$.} The convergence 
\[
\mathrm{err}_{train}(t)=\frac{1}{n}\|e^{-tS}\tilde Y\|\longrightarrow 0 ,\quad \mbox{as $t\to\infty$},
\]
for all possible input is equivalent to the matrix convergence 
\[
e^{-tS}\longrightarrow 0, \quad \mbox{as $t\to\infty$}.
\]
Note indeed that the centered input $\tilde Y$ belongs to the space orthogonal to $1_n$ but the direction $1_n$ is well-controlled since Assumption~\ref{ass:cst} implies $S1_n=1_n$ and hence $e^{-tS}1_n=e^{-t}1_n\to 0$. Finally, the convergence $e^{-tS} \longrightarrow 0$ is equivalent to the fact that all the (complex) eigenvalues of $S$ have a positive real part (see for example \cite{bellman1997} and the references therein).

\medskip \noindent \textit{Point $ii)$.}
The relation 
\begin{align}
\mathbb{E}[\mathrm{err}_{train}(t)]&=\frac{1}{n}\mathbb{E}[\|[R_t]\|^2]\nonumber \\
&= \frac{1}{n}\|\mathbb{E}[R_t]\|^2+\frac{1}{n}\mathrm{Trace}(\mathrm{Var}(R_t)).\label{eq:biais-variance}
\end{align}
yields the decomposition into squared bias and variance.  The vector of residuals is $R_t=e^{-t S}\tilde Y$ where  $\tilde Y=(Y_i-\bar Y_n)_{1\leq i\leq n}$ has expectation and variance 
\begin{align*}
\mathbb{E}[\tilde Y]&=f-\bar f 1_n=\tilde f,\\
\mathrm{Var}[\tilde Y]&=\sigma^2\left(I-\frac{1}{n}1_n1_n^T\right)=\sigma^2 J.
\end{align*}
We deduce the squared bias and variance
\begin{align*}
\mathrm{bias}^2(t)&=\frac{1}{n}\big\|\mathbb{E}[R_t]\big\|^2=\frac{1}{n}\big\|e^{-tS}\tilde f\big\|^2,\\
\mathrm{var}_{train}(t)&=\frac{1}{n}\mathrm{Trace}(\mathrm{Var}[R_t])=\frac{\sigma^2}{n}\mathrm{Trace}\big(e^{-tS} J e^{-tS^T}\big)
\end{align*}
and Equation~\eqref{eq:expected-training-error} follows from Equation~\eqref{eq:biais-variance}.

\medskip \noindent 
\textit{Point $iii)$.} When $S$ is symmetric, we use the  decomposition
$S=\sum_{i=1}^n \mu_i u_iu_i^T$ which implies 
\[
e^{-tS}\tilde f=\sum_{i=1}^n e^{-t\mu_i}(u_i^T \tilde f)u_i
\]
and 
\[
\mathrm{bias}^2(t)=\frac{1}{n}\|e^{-tS}\tilde f\|^2=\frac{1}{n}\sum_{i=1}^n e^{-2t\mu_i}(u_i^T \tilde f)^2.
\]
On the other hand, using furthermore the relations $S^T=S$, $\mathrm{Trace}(AB)=\mathrm{Trace}(BA)$ and $J^2=J$, we get
\begin{align*}
\mathrm{var}_{train}(t)&=\frac{\sigma^2}{n}\mathrm{Trace}\Big( e^{-tS}Je^{-tS^T}\Big)=\frac{\sigma^2}{n}\mathrm{Trace}\Big( J e^{-2tS}J^T\Big)\\
&=\frac{\sigma^2}{n}\sum_{i=1}^n e^{-2t\mu_i}\mathrm{Trace}\Big(Ju_iu_i^T J^T\Big)
= \frac{\sigma^2}{n}\sum_{i=1}^n e^{-2t\mu_i}\|Ju_i\|^2.
\end{align*}
\end{proof}

\begin{proof}[Proof of Proposition~\ref{prop:test-error-2}]
\textit{Point $i)$.} The proof is similar to the proof of point $ii)$ in Proposition~\ref{prop:train-error} and we give the main lines only.  Equation~\eqref{eq:err-test-2} is equivalent to
\[
\mathrm{err}_{test}(t)=\frac{1}{n}\big\| R_t'\|^2\quad \mbox{with }R_t'=(Y_i'-\hat F_t(x_i))_{1\leq i\leq n},
\]
so that $\mathbb{E}[\mathrm{err}_{test}(t)]=\mathrm{bias}^2(t)+\mathrm{var}_{test}(t)$ with
\begin{align*}
\mathrm{bias}^2(t)&=\frac{1}{n}\big\|\mathbb{E}[R_t']\|^2,\\
\mathrm{var}_{test}(t)&=\frac{1}{n}\mathrm{Trace}(\mathrm{Var}(R_t')).
\end{align*}
Since $\mathbb{E}[R_t']=\mathbb{E}[R_t]=e^{-tS}\tilde f$, the squared bias is the same as in Proposition~\ref{prop:train-error}. Finally, the formula for the variance follows from the relation
\[
R_t'=Y'+\bar{Y}_n1_n +(e^{-tS}-I)\tilde Y
\]
where $Y'=(Y'_i)_{1\leq i\leq n}$, $\bar{Y}_n1_n$ and $(e^{-tS}-I)\tilde Y$ are uncorrelated with variance $\sigma^2 I$, $\sigma^2n^{-1}1_n1_n^T$ and $\sigma^2(I-e^{-tS})J(I-e^{-tS})^T$ respectively.

\medskip\noindent\textit{Point $ii)$.} The formula are proved in the same way as the formulas of point $iii)$ in Proposition~\ref{prop:train-error} and we omit the proof for the sake of brevity. The claimed properties of the squared bias and variance are straightforward. To prove the monotonicity of the expected test error near the origin and at infinity, it is enough to compute the derivative and prove that it is negative near $0$ and positive near infinity. The limit $2\sigma^2$  relies on the fact that $\sum_{i=1}^n \|Ju_i\|^2=n-1$ because $J$ is an orthogonal projection of rank $n-1$ and $(u_i)_{1\leq i\leq n}$ and orthonormal basis. Details are left to the reader.
\end{proof}

\begin{proof}[Proof of Proposition~\ref{prop:test-error}]
Conditionally on $X'=x'$, we have the decomposition
\begin{align*}
\mathbb{E}[(Y'-\hat F_t(X'))^2\mid X'=x']&=\sigma^2+\mathbb{E}[(f(x')-\hat F_t(x'))^2]\\
&= \sigma^2+(f(x')-\mathbb{E}[F_t(x')])^2+\mathrm{Var}[F_t(x')].
\end{align*}
By Proposition~\ref{prop:vanishing-learning-rate},  \[
\hat F_t(x')=\bar Y_n+g(x')^T S^{-1}\left(I_n-\mathrm{e}^{-tS}\right)\tilde Y,
\]
with expectation and variance given by
\begin{align*}
\mathbb{E}[\hat F_t(x')]&= \bar f+ g(x')^T S^{-1}\left(I_n-\mathrm{e}^{-tS}\right)\tilde f\\
\mathrm{Var}[\hat F_t(x')]&= \frac{\sigma^2}{n}+g(x')^T S^{-1}\left(I_n-\mathrm{e}^{-tS}\right)J_n\left(I_n-\mathrm{e}^{-tS}\right)S^{-1}g(x').
\end{align*}
We deduce
\begin{align*}
&\mathbb{E}[(Y'-\hat F_t(X'))^2\mid X'=x']=\sigma^2+(f(x')-\bar f- \tilde f^T S^{-1}\left(I_n-\mathrm{e}^{-tS}\right)g(x'))^2 \\
&\quad\quad +\frac{\sigma^2}{n}+g(x')^T S^{-1}\left(I_n-\mathrm{e}^{-tS}\right)J_n\left(I_n-\mathrm{e}^{-tS}\right)S^{-1}g(x').
\end{align*}
Integrating with respect to $x'$, we obtain the announced value of the test error.
\end{proof}

\subsection{Proofs for Section~\ref{sec:SGB}}

\begin{proof}[Proof of Proposition~\ref{prop:randomized-setting-fdd}]
As a preliminary, we state a Markov property of the stochastic boosting algorithm. For $x\in[0,1]^p$, we note $\mathbf{x}=(x_i)_{1\leq i\leq n+1}$ with the convention $x_{n+1}=x$ and also $\hat F_m^\lambda(\mathbf{x})=(\hat F_m^\lambda(x_i))_{1\leq i\leq n+1}$. We observe that $(\hat F_m^\lambda(\mathbf{x}))_{m\geq 0}$ is a time homogeneous Markov chain with values in $\mathbb{R}^{n+1}$. 
Indeed, the recursive relation~\eqref{eq:stochastic-boosting-recursion} implies  that $\hat F_{m+1}^\lambda(\mathbf{x})$ depends  only of $(\hat F_{m}^\lambda(x_i))_{1\leq i\leq n}$, and $\xi_{m+1}$. The time homogeneous Markov property follows since $\hat F_{m}^\lambda(\bf x)$ contains $(\hat F_{m}^\lambda(x_i))_{1\leq i\leq n}$ in its first $n$ components  and  $\xi_{m+1}$ is independent on the past $\hat F_{0}^\lambda(\mathbf{x}),\ldots, \hat F_{m}^\lambda(\mathbf{x})$. 

\medskip\noindent
\textit{Point $i)$.} Taking conditional expectation, Equation~\eqref{eq:stochastic-boosting-recursion} implies
\begin{equation}\label{eq:conditional-expectation}
\mathbb{E}_\xi[ \hat F_{m+1}^\lambda(\mathbf{x}) \mid \hat F_{m}^\lambda(\mathbf{x})]=\hat F_m^\lambda(\mathbf{x})+\lambda \sum_{i=1}^n (Y_i-\hat F_m^\lambda(x_i))\bar g_i(\mathbf{x}),\quad  m\geq 0,
\end{equation}
with $\bar g_j(\mathbf{x})=(\bar g_j(x_i))_{1\leq i\leq n+1}$. We deduce 
\[
\mathbb{E}_\xi[ \hat F_{m+1}^\lambda(\mathbf{x}) ]=\mathbb{E}_\xi[ \hat F_{m}^\lambda(\mathbf{x}) ]+\lambda \sum_{i=1}^n (Y_i-\mathbb{E}_\xi[ \hat F_m^\lambda(x_i)])\bar g_i(\mathbf{x}),\quad  m\geq 0.
\]
Considering component $n+1$, we see that the functions $x\mapsto \mathbb{E}_\xi[ \hat F_{m+1}^\lambda(x)]$ satisfy the recursive relation~\eqref{eq:boosting-recursion} where the linear base learner is given by~\eqref{eq:linear-assumption} with $g_j$ replaced by $\bar g_j$. Proposition~\ref{prop:weight-formula} then yields the explicit form for $\mathbb{E}_\xi[ \hat F_{m+1}^\lambda(\mathbf{x})]$ stated in Equation~\eqref{eq:expectation}. 

\medskip\noindent
\textit{Point $ii)$.} In order to compute the variance of $\hat F_m^\lambda(x_j)$, $1\leq j\leq n+1$, we use the recursive relation~\eqref{eq:boosting-recursion} together with the variance decomposition  
\[
\mathrm{Var}_\xi[\hat F_{m+1}^\lambda(x_j)]=\mathrm{Var}_\xi[\mathbb{E}_\xi[\hat F_{m+1}^\lambda(x_j) \mid \hat F_{m}^\lambda(\mathbf{x})]]+\mathbb{E}_\xi[\mathrm{Var}_\xi[\hat F_{m+1}^\lambda(x_j) \mid \hat F_{m}^\lambda(\mathbf{x})]].
\]
In view of Equation~\eqref{eq:conditional-expectation}, the first term satisfies
\begin{align*}
&\mathrm{Var}_\xi[\mathbb{E}_\xi[\hat F_{m+1}^\lambda(x_j) \mid \hat F_{m}^\lambda(\mathbf{x})]]\\
&=\mathrm{Var}_\xi\Big[\hat F_m^\lambda(x_j)+\lambda \sum_{i=1}^n (Y_i- \hat F_m^\lambda(x_i))\bar g_i(x_j) \Big]\\
&=\mathrm{Var}_\xi[\hat F_m^\lambda(x_j)]+\lambda^2\sum_{1\leq i,k\leq n} \bar g_i(x_j)\bar g_k(x_j) \mathrm{Cov}_\xi[\hat F_m^\lambda(x_i),\hat F_m^\lambda(x_k)]\\
&\quad - 2\lambda\sum_{i=1}^n \bar g_i(x_j) \mathrm{Cov}_\xi[\hat F_m^\lambda(x_j),\hat F_m^\lambda(x_i)]  \\
&\leq \left(1+\lambda M_1\right)^2\max_{1\leq i\leq n+1} \mathrm{Var}_\xi[\hat F_{m}^\lambda(x_i)],
\end{align*}
where $M_1$ is given by Equation~\eqref{eq:M_1}.
In the last inequality, we use the Cauchy-Schwartz inequality to upper bound the covariances. We next provide an upper bound for the second term in the variance decomposition. We have
\begin{align*}
&\mathrm{Var}_\xi[\hat F_{m+1}^\lambda(x_j) \mid \hat F_{m}^\lambda(\mathbf{x})]\\
&=\lambda^2 \mathrm{Var}_\xi\Big[\sum_{i=1}^n (Y_i- \hat F_m^\lambda(x_i)) g_i(x_j)\mid \hat F_{m}^\lambda(\mathbf{x}) \Big]\\
&= \lambda^2\sum_{1\leq i,k\leq n}(Y_i-\hat F_m^\lambda(x_i))(Y_k- \hat F_m^\lambda(x_k))\mathrm{Cov}_\xi[g_i(x_j),g_k(x_j)]\\
&\leq \lambda^2 M_2\sum_{i=1}^n (Y_i- \hat F_m^\lambda(x_i))^2,
\end{align*}
where $M_2$ is defined in~\eqref{eq:M_2}.
The last line relies on the inequality $u^T\Sigma u\leq \rho(\Sigma)\|u\|^2$, where $u\in\mathbb{R}^n$, $\Sigma\in\mathbb{R}^{n\times n}$ is a non negative symmetric  matrix and $\rho(\Sigma)$ denotes its spectral radius, i.e. its largest eigenvalue. We apply this inequality with $u=(Y_i-\hat F_m^\lambda(x_i))_{1\leq i\leq n}$ and $\Sigma=(\mathrm{Cov}_\xi[g_i(x_j),g_k(x_j)])$ and we use the fact that $\rho(\Sigma)\leq \mathrm{Trace}(\Sigma) $. We deduce that the second term in the variance decomposition is upper bounded by
\begin{align*}
&\mathbb{E}_\xi[\mathrm{Var}_\xi[\hat F_{m+1}^\lambda(x_j) \mid \hat F_{m}^\lambda(\mathbf{x})]]\\
&\leq \lambda^2 M_2 \sum_{i=1}^n \mathbb{E}_\xi[(Y_i- \hat F_m^\lambda(x_i))^2]\\
&\leq n\lambda^2 M_2 \max_{1\leq i\leq n+1} \mathrm{Var}_\xi[\hat F_{m}^\lambda(x_i)]+\lambda^2 M_2\sum_{i=1}^n(Y_i-\mathbb{E}_\xi[\hat F_m^\lambda(x_i)])^2.
\end{align*}
Collecting the two terms of the variance decomposition, we get
\[
\mathrm{Var}_\xi[\hat F_{m+1}^\lambda(x_j)]\leq (1+\alpha) a_m + \beta b_m
\]
with $\alpha=2\lambda M_1+\lambda^2 M_1^2+n\lambda^2 M_2$, $\beta=\lambda^2 M_2$, $a_m=\max_{1\leq i\leq n+1} \mathrm{Var}_\xi[\hat F_{m}^\lambda(x_i)]$ and $b_m=\sum_{i=1}^n(Y_i-\mathbb{E}_\xi[\hat F_m^\lambda(x_i)])^2$. Taking the maximum over $j=1,\ldots, n+1$, note that the sequence $(a_m)_{m\geq 0}$ satisfies
\[
a_{m+1}\leq (1+\alpha ) a_m + \beta b_m,\quad m\geq 0.
\]
By the discrete Gronwall lemma or a straightforward induction, we deduce
\[
a_m\leq (1+\alpha )^{m+1} a_0+\beta (1+\alpha)^m\sum_{k=0}^{m} b_k,\quad m\geq 0.
\]

\noindent We use now Equations~\eqref{eq:expectation} and~\eqref{eq:weight-formula} to show that 

\begin{align*}
b_k&\leq 2\sum_{i=1}^n (Y_i-\bar Y_n)^2+2nM_1^2\left(\max_{1\leq i\leq n} \left| w_{k,i}^\lambda\right|\right)^2\\
&\leq 2n \|\tilde Y\|_\infty^2+2nM_1^2\left(\sum_{j=1}^k\binom{k}{j}\lambda^j\|S\|_\infty^{j-1}\|\tilde Y\|_\infty\right)^2\\
&\leq 2n \|\tilde Y\|_\infty^2+2nM_1^2\|\tilde Y\|_\infty^2\left(\lambda k\sum_{j=1}^k\frac{(\lambda k \|S\|_\infty)^{j-1}}{(j-1)!}\right)^2\\
&\leq 2n\|\tilde Y\|_\infty^2 \left\lbrace 1+ (\lambda k M_1)^2\mathrm{e}^{2\lambda k \|S\|_\infty}\right\rbrace.
\end{align*}
This implies that 
\[
\sum_{k=0}^m b_k\leq 2(m+1)n\|\tilde Y\|_\infty^2\left\lbrace 1+ (\lambda m M_1)^2\mathrm{e}^{2\lambda m \|S\|_\infty}\right\rbrace.
\]
Therefore, as $a_0=0$, it follows that 

\begin{align*}
\max_{1\leq i\leq n+1} \mathrm{Var}_\xi[\hat F_{m}^\lambda(x_i)]&\leq 2(m+1)\beta (1+\alpha)^mn\|\tilde Y\|_\infty^2 \left\lbrace 1+ (\lambda m M_1)^2\mathrm{e}^{2\lambda m \|S\|_\infty}\right\rbrace.
\end{align*}
Since, for $\lambda <1$, $\alpha\leq (2M_1+M_1^2+(n+1)M_2)\lambda$ and $2\beta\leq (2M_1+M_1^2+(n+1)M_2)\lambda^2$, the expected upper bound for the variance of the stochastic boosting output $\hat{F}_m^\lambda(x)$ is obtained.
\end{proof}

\begin{proof}[Proof of Corollary~\ref{cor:vanishing-limit-stochastic}]
First observe that
\[
\lim_{\lambda\to 0} \mathbb{E}_\xi[\hat F_{[t/\lambda]}^\lambda(x)]
= \bar{Y}_n+\sum_{i=1}^n w_{t,i} \bar g_i(x)=\hat F_t(x).
\]
This is a straightforward consequence of Equation~\eqref{eq:expectation} and of the weight convergence $w_{[t/\lambda]}^\lambda \to w_t$ stated in Proposition~\ref{prop:vanishing-learning-rate}. Together with the convergence of the variance $\mathrm{Var}_\xi[\hat F_{[t/\lambda]}^\lambda(x)]\to 0$ deduced from  point $ii)$ of Proposition~\ref{prop:randomized-setting-fdd}, this yields the convergence in quadratic mean $\hat F_{[t/\lambda]}^\lambda(x)\to  \hat F_t(x)$  as $\lambda\to 0$.
\end{proof}

\begin{proof}[Proof of Theorem~\ref{thm:Markov-functional}] \textit{Point $i)$.} The recursive relation \eqref{eq:stochastic-boosting-recursion} can be rewritten as
\[
\hat F_{m+1}^\lambda=T(\hat F_{m}^\lambda,\xi_{m+1}),\quad m\geq 0,
\]
with $T:L^\infty\times \Xi\to L^\infty$ defined by
\[
T(F,\xi)=F(\cdot)+\lambda \sum_{i=1}^n (Y_i-F(x_i))g_i(\cdot,\xi).
\]
Since $(\xi_m)_{m\geq 1}$ is i.i.d. and independent of $(x_i)_{1\leq i\leq n}$, $(Y_i)_{1\leq i\leq n}$ and $\hat F_0^\lambda=\bar Y_n$, this implies that $(\hat F_m^\lambda)_{m\geq 0}$ is a time homogeneous Markov chain. 

\medskip\noindent
 \textit{Point $ii)$.}  Note that the Markov chain $(\hat F_m^\lambda)_{m\geq 0}$ remains in the finite dimensional subspace 
\[
\mathcal{F}=\mathrm{span}(1,g_i(\cdot,\xi); 1\leq i\leq n, \xi\in\Xi)\subset L^\infty
\]
and that the Markov property stated in point $i)$ remains true if we replace $L^\infty$ by the subspace $\mathcal{F}$. We apply Theorem 11.2.3 in \cite{SV06}, page 272, to the Markov chain $(\hat F_m^\lambda)_{m\geq 0}$ on $\mathcal{F}$. Let $f\in\mathcal{F}$ and consider the local drift  and volatility  of the Markov chain at $f$ defined respectively by
\begin{align*}
b_\lambda(f)&=\lambda^{-1}\mathbb{E}_\xi[\hat F_{m}+1^\lambda-\hat F_{m}^\lambda\mid \hat F_{m}^\lambda=f]\\
a_\lambda(f)&=\lambda^{-1}\mathbb{E}_\xi[(\hat F_{m+1}^\lambda-\hat F_{m}^\lambda)(\hat F_{m+1}^\lambda-\hat F_{m}^\lambda)^T\mid \hat F_{m}^\lambda=f]
\end{align*}
where $f$ is  implicitly identified  with its vector of coordinates in some basis so that the product $ff^T$ makes sense. 

Given $\hat F_{m}^\lambda=f$, we have
\[
\hat F_{m+1}^\lambda-\hat F_{m}^\lambda=\lambda\sum_{i=1}^n (Y_i-f(x_i)) g_i(\cdot,\xi_{m+1}).
\]
We deduce that the local drif is given by 
\[
b_\lambda(f) =  \sum_{i=1}^n (Y_i-f(x_i))\bar g_i
\]
and does not depend on $\lambda$, i.e. $b_\lambda(f)=b(f)$. 

To deal with the local volatility $a_\lambda(f)$, note first that there exists a constant $C>0$ such that the matrix $ff^T$ has all its coefficients bounded by $C\|f\|_\infty^2$.
This is a consequence of the equivalence of norms on the finite dimensional space $\mathcal{F}$: the norm $\|\cdot\|_\infty$ is equivalent to the norm of the vector representing $f$ in the basis implicitly used for computing $ff^T$. We can thus bound the coefficients of the local volatility matrix by 
\begin{align*}
&\lambda C\mathbb{E}_\xi \Big\|\sum_{i=1}^n (Y_{i}-f(x_i))g_i(\cdot,\xi_{m+1})\Big\|_\infty^2\\
&\leq \lambda CM^2\Big(\max_{1\leq i\leq n} \left|Y_i-f(x_i)\right|\Big)^2,
\end{align*}
where $M=\max_{\xi\in\Xi}\|\sum_{i=1}^n |g_i(\cdot,\xi)|\|_\infty$ is finite because $\Xi$ is finite. We deduce
\[
a^\lambda(f)\longrightarrow a(f) \equiv 0 \quad \mbox{uniformly on compact sets as $\lambda\to 0$.}
\]
The limit functions $a$ and $b$ are continuous. Since $a\equiv 0$ and $b$ is an affine function, the associated martingale problem has exactly one solution starting from any point, see \cite{SV06} Lemma 6.1.4 page 140 or Theorem 6.3.4 page 152. In fact, because the limit volatility $a\equiv 0$ is vanishing, the solution of the martingale problem is the solution of the  differential equation on $\mathcal{F}$
\[
f'(t)=b(f(t))=\sum_{i=1}^n (Y_i-f(x_i))\bar g_i,\quad t\geq 0.
\]
This is exactly the differential Equation~\eqref{eq:EDO} and we have proved in Theorem~\ref{thm:EDO} that it has a unique solution with initial condition $f(0)=\bar Y_n$. Then, Theorem 11.2.3 in \cite{SV06} implies that the continous processes defined by interpolation
\[
\tilde{F}_{t}^\lambda=(1-\{t/\lambda\})\hat F_{[t/\lambda]}^\lambda+\{t/\lambda\} \hat F_{[t/\lambda]+1}^\lambda,\quad t\geq 0,
\]
converge in  distribution in the space of continuous functions $\mathbb{C}([0,\infty),\mathcal{F})$
\begin{equation}\label{eq:conv-C}
(\tilde{F}_{t}^\lambda )_{t\geq 0} \stackrel{d}\longrightarrow (\hat F_t)_{t\geq 0}.
\end{equation}
The notation $\{u\}=u-[u]$ stands for the fractional part of a real number. Finally, the convergence in distribution~\eqref{eq:conv-functional} in the Skorokhod space $\mathbb{D}([0,\infty),\mathcal{F})$ follows by a standard discretization argument. It holds 
\[
(\hat F^\lambda_{[t/\lambda]})_{t\geq 0}=\Psi_\lambda((\tilde{F}^\lambda_{t})_{t\geq 0})
\]
where  $\Psi_\lambda:\mathbb{C}([0,\infty),\mathcal{F})\longrightarrow \mathbb{D}([0,\infty),\mathcal{F})$ is the discretization functional
\[
\Psi_\lambda((f_t)_{t\geq 0})=(f_{\lambda[t/\lambda]})_{t\geq 0}.
\] 
The functional $\Psi_\lambda$ satisfies the following property: for all converging sequence  $(f_t^\lambda)\to (f_t)$ in $\mathbb{C}([0,\infty),\mathcal{F})$  as $\lambda\to 0$, it holds $\Psi_\lambda((f_t^\lambda))\to (f_t)$ in $\mathbb{D}([0,\infty),\mathcal{F})$ as $\lambda\to 0$. Together with the convergence~\eqref{eq:conv-C}, this implies the convergence~\eqref{eq:conv-functional} by the generalized continuous mapping theorem \citep[Theorem 2.7]{B99}.
\end{proof}

\bibliographystyle{apalike}

\end{document}